\documentclass{article}
\usepackage{fullpage,amsmath,amssymb,url} 
\usepackage{graphicx}

\def\QED{\ensuremath{{\square}}}
\def\markatright#1{\leavevmode\unskip\nobreak\quad\hspace*{\fill}{#1}}
\newenvironment{proof}
 {\begin{trivlist}\item[\hskip\labelsep{\bf Proof.}]}
 {\markatright{\QED}\end{trivlist}}

\def\tildeD{{\tilde{D}}}
\def\calD{\mathcal{D}}
\def\calC{\mathcal{C}}
\def\inner#1#2{{\langle #1,#2\rangle}}

\def\dx{\mathrm{d}x}
\def\dmu{\mathrm{d}\mu}

\def\dP{\mathrm{d}P}
\def\KL{\mathrm{KL}}
\def\kl{\mathrm{kl}}
\def\Var{\mathrm{Var}}

\def\lse{\mathrm{lse}}
\def\diag{\mathrm{diag}}
\def\dom{\mathrm{dom}}
\def\Num{\mathrm{Num}}
\def\Den{\mathrm{Den}}
\def\defeq{{=:}}

\def\ekl{\mathrm{ekl}}
\def\Var{\mathrm{Var}}
\def\eqdef{:=}
\def\eqnota{:=:}

\def\tildeF{\tilde{F}}
\def\tildeG{\tilde{G}}

\def\hatG{\hat{G}}

\def\bbR{\mathbb{R}}
\def\bbN{\mathbb{N}}

\def\calS{\mathcal{S}}
\def\calE{\mathcal{E}}
\def\calX{\mathcal{X}}
\def\calM{\mathcal{M}}

\def\st{{\ :\ }}

\def\YES{$\checkmark$} 
\def\NO{$\times$}
\def\eKL{\mathrm{eKL}}

\newtheorem{lemma}{Lemma}
\newtheorem{theorem}{Theorem}
\newtheorem{definition}{Definition}
\newtheorem{remark}{Remark}
\newtheorem{property}{Property}
\newtheorem{fact}{Fact}
\newtheorem{assumption}{Assumption}

\title{Monte Carlo Information Geometry: The dually flat case}

\author{
Frank Nielsen$^1$\footnote{E-mail: {\tt Frank.Nielsen@acm.org}} \and 
Ga\"etan Hadjeres$^2$\footnote{E-mail: {\tt gaetan.hadjeres@sony.com}}
}

\date{$^1$ Sony Computer Science Laboratories Inc, Tokyo, Japan\\
$^2$ Sony Computer Science Laboratories, Paris, France}

\begin{document}
\maketitle
 
\begin{abstract}
Exponential families and mixture families are parametric probability models that can be geometrically studied as smooth statistical manifolds with respect to any statistical divergence like the Kullback-Leibler (KL) divergence or the Hellinger divergence.
When equipping a statistical manifold with the KL divergence, the induced manifold structure is  dually flat, and the KL divergence between distributions amounts to an equivalent Bregman divergence on their corresponding parameters.
In practice, the corresponding Bregman generators of
mixture/exponential families require to perform definite integral
calculus that can either be too time-consuming (for exponentially
large discrete support case) or even do not admit closed-form formula
(for continuous support case).
In these cases, the dually flat construction remains theoretical and cannot be used by information-geometric algorithms. 
To bypass this problem, we consider performing stochastic Monte Carlo
(MC) estimation of those integral-based mixture/exponential family
Bregman generators.
We show that, under natural assumptions, these MC
  generators are almost surely Bregman generators.
We define a series of dually flat information geometries, termed Monte
Carlo Information Geometries, that increasingly-finely approximate the
untractable geometry.
The advantage of this MCIG is that it allows a practical use of the
Bregman algorithmic toolbox on a wide range of probability
distribution families.
We demonstrate our approach with a clustering task on a  mixture family manifold.
\end{abstract}
 
\sloppy

\noindent {\bf Keywords}: Computational Information Geometry, Statistical Manifold, Dually flat information geometry, Bregman generator, Stochastic Monte Carlo Integration, Mixture family, Exponential Family, Clustering.

\section{Introduction}

We concisely describe the construction and properties of dually flat spaces~\cite{IG-2014,IG-2016} in~\S\ref{sec:dfsbg}, 
define the statistical manifolds of exponential families and mixture families in~\S\ref{sec:geosf},  and discuss about the computational tractability of Bregman algorithms in dually flat spaces in~\S\ref{sec:comptract}.

\subsection{Dually flat space: Bregman geometry}\label{sec:dfsbg}
A smooth (potentially asymmetric) distance $D(\cdot,\cdot)$ is called a {\em divergence} in information geometry~\cite{IG-2014,IG-2016}, and induces a differential-geometric dualistic structure~\cite{Eguchi-1992,IGDiv-2010,IG-2014,IG-2016}. 
In particular, a strictly convex and twice continuously differentiable $D$-dimensional real-valued function $F$, termed a {\em Bregman generator}, induces 
a dually connection-flat structure via a corresponding  Bregman Divergence (BD)~\cite{Bregman-2005} $B_F(\cdot,\cdot)$ given by:  
\begin{equation}
B_F(\theta_1:\theta_2) \eqdef F(\theta_1)-F(\theta_2)-\inner{\theta_1-\theta_2}{\nabla F(\theta_2)},
\end{equation}
where $\inner{y}{x} \eqdef y^\top x$ denotes the inner product, and $\nabla F(\theta)\eqdef (\partial_i F(\theta))_i$ denotes the gradient vector of partial first-order derivatives. We use the standard notational convention of information geometry~\cite{IG-2014,IG-2016}: 
$\partial_i \eqnota \frac{\partial}{\partial \theta^i}$ to indicate a contravariant vector~\cite{tensor-2011} $\theta=(\theta^i)_i$.
(The $\eqnota$ symbol means it is a notational convention equality, like $\sum_{i=1}^k x_i \eqnota x_1+\ldots x_k$. 
It differs from $a\eqdef b$ that denotes the symbol by of a quantity equality by definition.)

The Legendre-Fenchel transformation~\cite{LegendreIG-2010} :
\begin{equation}
F^*(\eta)= \sup_\theta \{ \inner{\theta}{\eta}-F(\theta) \},
\end{equation} 
is at the heart of the duality of  flat structures by defining two global affine coordinate systems: The {\em primal affine $\theta$-coordinate system} and the {\em dual affine $\eta$-coordinate system}, so that any point $P$ of the manifold $\calM$ can  either be accessed by its {\em primal} $\theta(P)$ coordinates or equivalently by its {\em dual} $\eta(P)$ coordinates. 
We can convert between  these two dual coordinates as follows:
\begin{eqnarray}
 \eta&=&\eta(\theta)=\nabla F(\theta)=(\partial_i F(\theta))_i,\\
  \theta&=&\theta(\eta)=\nabla F^*(\eta)=(\partial^i F^*(\eta))_i,
	\end{eqnarray}
	with reciprocal gradients $\nabla F^* \eqdef (\nabla F)^{-1}$. 
	We used the notational convention $\partial^i \eqnota \frac{\partial}{\partial \eta_i}$ that indicates the covariant vector~\cite{tensor-2011} $\eta=(\eta_i)_i$.

The metric tensor $g$ of the dually flat structure $(\calM,F)$  can either be expressed using the $\theta$- or $\eta$-coordinates using the Hessians of the potential functions~\cite{HIG-2007}:
\begin{eqnarray}
G(\theta)&=&\nabla^2 F(\theta),\\
G^*(\eta)&=&\nabla^2 F^* (\eta),
\end{eqnarray}
and defines a smooth bilinear form ${\inner{v}{v'}}_g$ on $\calM$ so that for two vectors $v,w$ of a tangent plane $T_P$, we have:
\begin{eqnarray}
{\inner{v}{v'}}_g &=& \theta(v)^\top G(\theta) \theta(w),\\
 &=& \eta(v)^\top G^*(\eta) \eta(w),
\end{eqnarray}
where $\theta(v)=(v^i)_i$ and $\eta(v)=(v_i)_i$ denote the contravariant coefficients and covariant coefficients of a vector $v$, respectively.
That is,  any vector $v\in T_P$ can be written either as $v=\sum_i v^ie_i$ or as $\sum_i v_i{e^*}^i$, where  $\{e_i\}_i$ and
 $\{{e^*}^i\}_i$ is a dual basis~\cite{tensor-2011} of the vector space structure of $T_P$.

Matrices $G(\theta)$ and $G^*(\eta)$ are symmetric positive definite (SPD, denoted by $G(\theta)\succ 0$ and $G^*(\eta)\succ 0$), and 
 they satisfy the Crouzeix identity~\cite{Crouzeix-1977}: 
\begin{equation}
G(\theta)   G^*(\eta) = I,
\end{equation}
where $I$ stands for the $D\times D$ identity matrix.
This indicates that at each tangent plane $T_P$, the dual coordinate systems are biorthogonal~\cite{refduality-2015} (with $\{e_i\}_i$ and
 $\{{e^*}^i\}_i$ forming a dual basis~\cite{tensor-2011} of the vector space structure of $T_P$):
	\begin{equation}
	\inner{e_i}{{e^*}^j}=\delta_i^j,
	\end{equation}
 with $\delta_i^j$ the Kr\"onecker symbol: $\delta_i^j=1$ if and only if (iff) $i=j$, and $0$ otherwise.
We have:
	\begin{eqnarray}
	\frac{\partial \eta_i}{\partial\theta^j}&=&g_{ij}(\theta)=\inner{e_i}{e_j},\\
\frac{\partial \theta^i}{\partial\eta_j}&=&g^{ij}(\eta)=\inner{{e^*}^i}{{e^*}^j}.\\
\end{eqnarray}

	The convex conjugate functions $F(\theta)$ and $F^*(\eta)$ are called {\em dual potential functions}, and define the global metric~\cite{HIG-2007}.

Table~\ref{tab:dfs} summarizes the differential-geometric structures of dually flat spaces.
Since Bregman divergences are {\em canonical divergences} of dually flat spaces~\cite{IG-2016}, the geometry of dually flat spaces is also  referred to the {\em Bregman geometry}~\cite{Dawid-2007} in the literature.

\begin{definition}[Bregman generator] \label{def:bg}
A Bregman generator is a strictly convex and twice continuously differentiable real-valued function $F:\bbR^D \rightarrow \bbR$.
\end{definition}

Let us cite the following well-known properties~\cite{Bregman-2005} of Bregman generators:

\begin{property}[Bregman generators are equivalent up to modulo affine terms] \label{prop:bdaffine}
The Bregman generator $F_2(\theta)=F_1(\theta)+\inner{a}{\theta}+b$ (with $a\in\bbR^D$ and $b\in\bbR$) yields the same Bregman divergence as the Bregman divergence induced by $F_1$, $B_{F_2}(\theta_1:\theta_2) =B_{F_1}(\theta_1:\theta_2) $, and therefore the same dually flat space $(\calM,F_2)\cong (\calM,F_1)$.
\end{property}

\begin{property}[Linearity rule of Bregman generators] \label{prop:bdcomprule}
Let $F_1, F_2$ be two Bregman generators and $\lambda_1, \lambda_2>0$. 
Then $B_{\lambda_1 F_1+\lambda_2 F_2}(\theta:\theta')= \lambda_1B_{F_1}(\theta:\theta')+\lambda_2 B_{F_2}(\theta:\theta')$.
\end{property}

\begin{table}
\centering
{\small 
\begin{tabular}{|l||l|l|}\hline
Manifold $(\calM,F)$ & Primal structure & Dual structure\\ \hline\hline
Affine coordinate system & $\theta(\cdot)$ & $\eta(\cdot)$\\ 
Conversion $\theta\leftrightarrow\eta$ & $\theta(\eta)=\nabla F^*(\eta)$ & $\eta(\theta)=\nabla F(\theta)$\\ \hline 
Potential function &  $F(\theta)=\inner{\theta}{\nabla F(\theta)} - F^*(\nabla F(\theta)) $ & $F^*(\eta)=\inner{\eta}{\nabla F^*(\eta)}-F(\nabla F^*(\eta))$\\ \hline
Metric tensor $g$ & $G(\theta)=\nabla^2 F(\theta)$ & $G^*(\eta)=\nabla^2 F^*(\eta)$\\ \hline
& $g_{ij}=\partial_i\partial_j F(\theta)$ & $g^{ij}=\partial^i\partial^j F^* (\eta)$\\
Geodesic ($\lambda\in [0,1]$) & $\gamma(P,Q)=\{(PQ)_\lambda=(1-\lambda)\theta(P)+\lambda\theta(Q)\}_\lambda$ & $\gamma^*(P,Q)=\{(PQ)_\lambda^*=(1-\lambda)\eta(P)+\lambda\eta(Q) \}_\lambda$\\ \hline
\end{tabular}
}

\caption{Overview of the dually differential-geometric structure $(\calM,F)$ induced by a Bregman generator $F$.  Notice that if $F$ and $\nabla F^*$ are available in closed-form then so are $\nabla F$ and $F^*$. \label{tab:dfs}}

\end{table}

In practice, the algorithmic toolbox in dually flat spaces (e.g., clustering~\cite{Bregman-2005}, minimum enclosing balls~\cite{SEB-2008}, hypothesis testing~\cite{HT-2013} and Chernoff information~\cite{Chernoff-2013}, Voronoi diagrams~\cite{BVD-2007,BVD-2010}, proximity data-structures~\cite{BVT-2009,BBT-2009}, etc.) can be used whenever the dual Legendre convex conjugates $F$ and $F^*$ are both available in closed-form (see type~1 of Table~\ref{tab:level}). In that case, both the primal $\gamma(P,Q)\eqdef \{(PQ)_\lambda\}_\lambda$ and dual $\gamma^*(P,Q)\eqdef \{(PQ)_\lambda^*\}_\lambda$  geodesics are available in closed form.
These dual geodesics can either be  expressed using the $\theta$ or $\eta$-coordinate systems as follows:

\begin{equation}
(PQ)_\lambda=\left\{
\begin{array}{l}
\theta((PQ)_\lambda) = \theta(P)+\lambda (\theta(Q)-\theta(P)),\\
\eta((PQ)_\lambda) = \nabla F(\theta((PQ)_\lambda))  = \nabla F(\nabla F^*(\eta(P))+\lambda (\nabla F^*(\eta(Q))-\nabla F^*(\eta(P)))),
\end{array}
\right.
 \end{equation}

\begin{equation}
(PQ)_\lambda^*=\left\{
\begin{array}{l}
\eta((PQ)_\lambda^*)  = \eta(P)+\lambda (\eta(Q)-\eta(P)),\\
\theta((PQ)_\lambda^*) = \nabla F^*(\eta((PQ)_\lambda^*)) =  \nabla F^*(\nabla F(\theta(P))+\lambda (\nabla F(\theta(Q))-\nabla F(\theta(P))))
\end{array}
\right.
\end{equation}
That is, the primal geodesic corresponds to a straight line in the primal coordinate system while the dual geodesic is a straight line in the dual coordinate system.
However, in many interesting cases, the convex generator $F$ or its dual $F^*$ (or both) are not available in closed form or are computationally intractable, and the above Bregman toolbox cannot be used.
Table~\ref{tab:bregalg} summarizes the closed-form formulas required to execute some fundamental clustering algorithms~\cite{Bregman-2005,Bregman1DClustering-2014,Bregman1DClustering-2017} in a Bregman geometry. 

\begin{table}
\centering
\begin{tabular}{|l||l|l|l|l|}\hline
Algorithm & $F(\theta)$ & $\eta(\theta)=\nabla F(\theta)$ & $\theta(\eta)=\nabla F^*(\eta)$ & $F^*(\eta)$\\ \hline\hline
Right-sided Bregman clustering & \YES & \YES & \NO & \NO  \\
Left-sided Bregman clustering &  \NO & \NO & \YES & \YES\\
Symmetrized Bregman centroid & \YES & \YES & \YES & \YES\\
Mixed Bregman clustering &   \YES & \YES & \YES & \YES\\\hline
Maximum Likelihood Estimator for EFs &   \NO & \NO & \YES & \NO\\\hline
Bregman soft clustering ($\equiv$ EM) &  \NO & \YES & \YES & \YES \\ \hline
\end{tabular}

\caption{Some fundamental Bregman clustering algorithms~\cite{Bregman-2005,Bregman1DClustering-2014,Bregman1DClustering-2017} (of the Bregman algorithmic toolbox) that illustrate which closed-form are required to be run in practice.\label{tab:bregalg}}

\end{table}

Let us notice that so far the points $P\in\calM$ in the dually flat manifold have no particular meaning, and that the dually flat space structure is generic, not necessarily related to a statistical flat manifold.
We shall now review quickly the dualistic structure of statistical manifolds~\cite{Lauritzen-1987}.

\subsection{Geometry of statistical manifolds}\label{sec:geosf}
Let $I_1(x;y)$ denote a {\em scalar divergence}.
A {\em statistical divergence}  between two probability distributions $P$ and $Q$, with Radon-Nikodym derivatives $p(x)$ and $q(x)$ with respect to (wrt)  a base measure $\mu$ defined on the support $\calX$, is defined as:
\begin{equation}
I(P:Q)=\int_{x\in\calX} I_1\left(p(x):q(x)\right) \dmu(x).
\end{equation}
A  statistical divergence is a measure of dissimilarity/discrimination that satisfies $I(P:Q)\geq 0$ with equality iff. $P=Q$ (a.e., reflexivity property) . 
For example, the Kullback-Leibler divergence is a statistical divergence:
\begin{equation}
\KL(P:Q)\eqdef \int_{x\in\calX} \kl(p(x):q(x)) \dmu(x),
\end{equation}
with corresponding scalar divergence:
\begin{equation}
\kl(x:y) \eqdef x\log\frac{x}{y}.
\end{equation}

The KL divergence between $P$ and $Q$ is also called the {\em relative entropy}~\cite{Cover-2012} because it is the difference of the 
{\em cross-entropy} $h^\times(P:Q)$ between $P$ and $Q$ with the {Shannon entropy} $h(P)$ of $P$:
 
\begin{eqnarray}
\KL(P:Q) &=& h^\times(P:Q)-h(P),\\
h^\times(P:Q) &\eqdef &\int_{x\in\calX} p(x)\log \frac{1}{q(x)} \dmu(x),\\
h(P) &\eqdef& \int_{x\in\calX} p(x)\log \frac{1}{p(x)} \dmu(x)= h^\times(P:P).
\end{eqnarray}

Thus we distinguish a statistical divergence from a parameter divergence by stating that   a statistical divergence is a separable divergence that is the definite integral on the support of a scalar divergence.

In information geometry~\cite{IG-2014,IG-2016}, we	equip a probability manifold $\calM=\{p(x;\theta)\st \theta\in\Theta\}$ with a {\em metric tensor} $g$ (for measuring angles between vectors and lengths of vectors in tangent planes) and a {\em pair of dual torsion-free connections} $\nabla$ and $\nabla^*$ (for defining parallel transports and geodesics) that are defined by their Christoffel symbols $\Gamma_{ijk}$ and  $\Gamma_{ijk}^*$. 
These geometric structures $(\calM,D)\eqdef (\calM,g_D,\nabla_D,\nabla^*_D)$ can be induced by {\em any smooth} $C^\infty$ divergence $D(\cdot:\cdot)$~\cite{Eguchi-1992,IGDiv-2010,IG-2014,IG-2016} as follows: 
\begin{eqnarray}
	g_{ij}(x)&=& \left. \frac{\partial^2}{\partial x_i\partial x_j} D(x:y)\right\vert_{y=x},\\
	\Gamma_{ijk}(x)&=& \left. -\frac{\partial^3}{\partial x_i\partial x_j\partial y_k} D(x:y)\right\vert_{y=x}.
\end{eqnarray}
The {\em dual  divergence} $D^*(p:q)\eqdef D(q:p)$ highlights the {\em reference duality}~\cite{refduality-2015}, and
the dual connection $\nabla^*$  is induced by the dual divergence $D^*(\cdot:\cdot)$ ($\nabla^*$ is defined by $\Gamma_{ijk}^*(x)=\left. -\frac{\partial^3}{\partial x_i\partial x_j\partial y_k} D^*(x:y)\right\vert_{y=x}$). 
Observe that the metric tensor is self-dual: $g^*=g$.


	\begin{table}%
\centering
\begin{tabular}{|l|c|r|}\hline
 & Exponential Family & Mixture Family\\ \hline
Density & $p(x;\theta)=\exp(\inner{\theta}{x}-F(\theta))$ & $m(x;\eta)=\sum_{i=1}^{k-1}\eta_i f_i(x)+c(x)$\\
&  & $f_i(x)=p_i(x)-p_0(x)$\\
Family/Manifold & $\calM=\{p(x;\theta)\ :\ \theta\in\Theta^\circ \}$ &  $\calM=\{m(x;\eta)\ :\ \eta\in H^\circ \}$ \\
Convex function ($\equiv ax+b$) & $F$: cumulant & $F^*$: negative entropy\\ \hline
Dual coordinates & moment $\eta=E[t(x)]$ &  $\theta^i=h^\times(p_0:m)-h^\times(p_i:m)$ \\ \hline
Fisher Information $g=(g_{ij})_{ij}$ & $g_{ij}(\theta)=\partial_i\partial_j F(\theta)$ & $g_{ij}(\eta)=\int_\calX \frac{f_i(x)f_j(x)}{m(x;\eta)}\dmu(x)$ \\
& $g=\Var[t(X)]$ & \\  \hline
& & $g_{ij}(\eta)=-\partial_i\partial_j h(\eta)$ \\
Christoffel symbol & $\Gamma_{ij,k}=\frac{1}{2}\partial_i\partial_j\partial_k F(\theta)$ &
$\Gamma_{ij,k}=-\frac{1}{2}\int_\calX \frac{f_i(x)f_j(x)f_k(x)}{m^2(x;\eta)} \dmu(x)$ \\ \hline
Entropy & $-F^*(\eta)$ & $-F^*(\eta)$\\
\hline
Kullback-Leibler divergence & $B_F(\theta_2:\theta_1)$ & $B_{F^*}(\eta_1:\eta_2)$\\
& $=B_{F^*}(\eta_1:\eta_2)$ & $=B_F(\theta_2:\theta_1)$\\ \hline
\end{tabular}
\caption{Characteristics of the dually flat geometries of Exponential Families (EFs) and Mixture Families (MFs).}
\label{tab:comparison}
\end{table}

Let us give some examples of parametric probability families and their statistical manifolds induced by the Kullback-Leibler divergence.

\subsubsection{Exponential family manifold (EFM)} 
	
	We start by a definition:
	
	\begin{definition}[Exponential family]\label{def:ef}
	Let $\mu$ be a prescribed  base measure and $t(x)$ a sufficient statistic vector.
	We can build  a corresponding exponential family:
	\begin{equation}
	\calE_{t,\mu}\eqdef\{p(x;\theta) \propto \exp(\inner{t(x)}{\theta})\}_\theta,
	\end{equation}
	where $p(x;\theta) \eqdef \frac{\dP(\theta)}{\dmu}(x)$.
	
	The densities are normalized by the cumulant function $F$:
	\begin{equation}\label{eq:cf}
	F(\theta)  \eqdef \log \left(\int_{x\in\calX} \exp(\inner{t(x)}{\theta})\dmu(x)\right),
	\end{equation}
	so that:
	\begin{equation}
	p(x;\theta) = \exp(\inner{t(x)}{\theta}-F(\theta)).
	\end{equation}
	Function $F$  is a Bregman generator on the natural parameter space:
	\begin{equation}
	\Theta\eqdef\left\{\theta : \int_{x\in\calX} \exp(\inner{t(x)}{\theta})\dmu(x)<\infty \right\}.
	\end{equation}
	If we add an extra carrier term $k(x)$ and consider the measure $\nu(x)\eqdef \frac{\mu(x)}{\exp(k(x))}$, we get the generic form of an exponential family~\cite{flashcards-2009}:
	\begin{equation}
	\calE_{t,k,\nu}\eqdef\left\{p(x;\theta) \propto \exp(\inner{t(x)}{\theta}+k(x)) \st \theta\in\Theta\right\}.
	\end{equation}
	\end{definition}
	
	We call function $F$ the {\em Exponential Family Bregman Generator}, or EFBG for short in the remainder.

	It turns out that $(\calE_{t,\mu},\KL,\nabla_\KL,\nabla_\KL^*)\cong (\calM,F)$ (meaning the information-geometric structure of the statistical manifold is isomorphic to the information-geometry of a dually flat manifold) so that:
	\begin{eqnarray}
	\KL(p(x;\theta_1):p(x;\theta_2)&=&B_F(\theta_2:\theta_1),\label{eq:KLBDEF}\\
	&=& B_{F^*}(\eta_1:\eta_2),\label{eq:KLBDEFD}
	\end{eqnarray}
	with $\eta=E_{p(x;\theta)}[t(x)]$ the dual parameter called the expectation parameter or moment parameter.

\subsubsection{Mixture family manifold (MFM)}
 Another important family of probability distributions are the mixture families:
	
	\begin{definition}[Mixture family]\label{def:mf}
	Given a set of $k$ prescribed statistical distributions $p_0(x),\ldots, p_{k-1}(x)$, all sharing the same support $\calX$ (say, $\bbR$), a
 {\em mixture family} $\calM$ of order $D=k-1$ consists of all {\em strictly convex combinations} of these component distributions~\cite{wmixture-2017,geowmixtures-2018}:
\begin{equation}
\calM \eqdef  \left\{ m(x;\eta)= \sum_{i=1}^{k-1} \eta_ip_i(x)+ (1-\sum_{i=1}^{k-1} \eta_i)p_0(x) \st \eta_i>0, \sum_{i=1}^{k-1} \eta_i<1 \right\}.
\end{equation}
\end{definition}
It shall be understood from the context that $\calM$ is a shorthand for $\calM_{p_0(x),\ldots, p_{D}}$.

	It turns out that $(\calM,\KL,\nabla_\KL,\nabla_\KL^*)\cong (\calM,G)$ so that:
	\begin{equation}\label{eq:KLBDMF}
\KL(m(x;\eta):m(x;\eta')) = B_G(\eta:\eta'),
\end{equation}
for the Bregman generator being the Shannon negative entropy (also called Shannon information):
\begin{equation}\label{eq:negent}
G(\eta)=-h(m(x;\eta))=\int_{x\in\calX} m(x;\eta)\log m(x;\eta) \dmu(x).
\end{equation}
	We call function $G$ the {\em Mixture Family Bregman Generator}, or MFBG for short in the remainder.
	
	For a mixture family, we prefer to use the notation $\eta$ instead of  $\theta$ for indexing the distribution parameters as it is customary in  textbooks of information geometry~\cite{IG-2014,IG-2016}.
One reason comes from the fact that the KL divergence between two mixtures amounts to a BD on their respective parameters (Eq.~\ref{eq:KLBDMF}) while the KL divergence between exponential family distributions is equivalent to a BD on the swapped order of their respective parameters (Eq.~\ref{eq:KLBDEF}). Thus in order to get the same order of arguments for the KL between  two  exponential family distributions, we need to use the dual Bregman divergence on the dual $\eta$ parameter, see Eq.~\ref{eq:KLBDEFD}.

\subsubsection{Cauchy family manifold (CFM)}

This example is merely given just to emphasize that probability families may neither be exponential nor mixture families.
	
	A Cauchy distribution has probability density  defined on the support $\calX=\bbR$ by:
	\begin{equation} 
p(x;\mu,\sigma) = \frac{1}{\pi \sigma \left(1+\left(\frac{(x-\mu}{\sigma}\right)^2\right)} .
\end{equation}
	The space of all Cauchy distributions:
	\begin{equation}
	\calC=\{p(x;\mu,\sigma) \st \mu\in\bbR, \sigma>0\}.
		\end{equation}
	is a location-scale family~\cite{KassVos-1997}. It is not an exponential family nor a mixture family.
	

	Table~\ref{tab:comparison} compares the dually flat structures of mixture families with exponential families.
	In information geometry, $(\calE_{t,k,\mu},\KL,\nabla_\KL,\nabla_\KL^*)=(\calE_{t,k,\mu},g,\nabla^e,\nabla^m)$ and 
	$(\calM,\KL,\nabla_\KL,\nabla_\KL^*)=(\calM,g,\nabla^m,\nabla^e)$
	where $g$ is the {\em Fisher information metric tensor} and $\nabla^e$ and $\nabla^m$ are the exponential and mixture connections, respectively.
	These connections are dual to each others, see~\cite{IG-2014}.

\subsection{Computational tractability of dually flat statistical manifolds}\label{sec:comptract}

The previous section explained the dually flat structures (i.e., Bregman geometry) of the exponential family manifold and of the mixture family manifold.
However these geometries may be purely theoretical as the Bregman generator $F$ may not be available in closed form so that the Bregman toolbox cannot be used in practice.
This work tackles this problem faced in exponential and mixture family manifolds by proposing the novel framework of
 {\em Monte Carlo Information Geometry} (MCIG).
MCIG approximates the untractable Bregman geometry by considering the Monte Carlo stochastic integration of the definite integral-based ideal Bregman generator.

\begin{table}
\centering
\begin{tabular}{|l||l|l|l|}\hline
Type & $F$ & $\nabla F^*$ & Example\\ \hline\hline
Type 1 & closed-form & closed-form & Gaussian (exponential) family \\
Type 2 & closed-form & not closed-form & Beta (exponential) family\\
Type 3 & comp. intractable &  not closed-form &  Ising family~\cite{BMFamily-2013} \\
Type 4 & not closed-form & not closed-form & Polynomial exponential family~\cite{PEF-2016}\\
Type 5 & not analytic & not analytic &  mixture family \\ \hline
\end{tabular}

\caption{A smooth and strictly convex function $F$ induces a dually flat structure: 
We classify those structures according to their computational tractability properties.\label{tab:level}}

\end{table}

But first, let us quickly review the five types of tractability  of Bregman geometry in the context of statistical manifolds by giving an illustrating family example for each type:

\begin{description}
\item[Type 1.]  $F$ and $\nabla F^*$ are both available in closed-form, and so are $\nabla F$ and $F^*$.
For example, this is the case of the  {\em the Gaussian exponential family}.
The normal distribution~\cite{flashcards-2009} has sufficient statistic vector $t(x)=(x,x^2)$ so that its log-normalizer is
\begin{equation}
F(\theta)=\log\left(\int_{-\infty}^{+\infty} \exp(\theta_1x+\theta_2 x^2) \dx\right).
\end{equation}

Since $\int_{-\infty}^{\infty} \exp(\theta_1x+\theta_2 x^2)=
\sqrt{\frac{\pi}{-\theta_2}}\exp(-\frac{\theta_1^2}{4\theta_2})$ for $\theta_2<0$, we find:
\begin{equation}
F(\theta)=\log\left(\int \exp(\theta_1x+\theta_2 x^2) \dx\right)
=-\frac{\theta_1^2}{4\theta_2}+\frac{1}{2}\log\frac{\pi}{-\theta_2}.
\end{equation}
This is in accordance with the direct canonical decomposition~\cite{flashcards-2009} of the density $p(x;\theta)=\exp(\inner{t(x)}{\theta}-F(\theta))$ of the normal density 
$p(x;\mu,\sigma)=\frac{1}{\sqrt{2\pi}\sigma}\exp(-\frac{(x-\mu)^2}{2\sigma^2})$.

\begin{remark}
When $F(\theta)$ can be expressed using the canonical decomposition of exponential families, this means that 
the definite integral $\log(\int \exp(\inner{t(x)}{\theta}+k(x))\dx)$ is available in closed form, and vice-versa.
\end{remark}

\item[Type 2.]  $F$ is available in closed form (and so is $\nabla F$) but $\nabla F^*$ is not available in closed form (and therefore $F^*$ is not available too).
This is for example the {\em  Beta exponential family}. 
A Beta distribution $\mathrm{Be}(\alpha,\beta)$ has density on support $x\in (0,1)$:
\begin{equation}
p(x;\alpha,\beta)=\frac{1}{B(\alpha,\beta)}x^{\alpha-1}(1-x)^{\beta-1},
\end{equation}
where $B(\alpha,\beta)=\frac{\Gamma(\alpha)\Gamma(\beta)}{\Gamma(\alpha+\beta)}$, and $(\alpha>0,\beta>0)$ are the shape parameters.
The Beta family of distributions is an exponential family with $\theta=(\alpha,\beta)$, $t(x)=(\log(x),\log(1-x))$, $k(x)=-\log(x)-\log(1-x)$ and
$F(\theta)=\log B(\theta_1,\theta_2)=\log\Gamma(\theta_1)+\log\Gamma(\theta_2)-\log\Gamma(\theta_1+\theta_2)$.
Note that we could also have chosen $\theta=(\alpha-1,\beta-1)$ and $k(x)=0$.
Thus $\nabla F(\theta)=(\psi(\theta_1)-\psi(\theta_1+\theta_2),\psi(\theta_2)-\psi(\theta_1+\theta_2))$ where $\psi(x)=\frac{\Gamma'(x)}{\Gamma(x)}$ is the digamma function. 
Inverting the gradient $\nabla F(\theta)=\eta$ to get $\eta=\nabla F^*(\theta)$ is not available in closed-form.\footnote{To see this, consider the digamma difference property: $f_\Delta(\theta)=\psi(\theta)-\psi(\theta+\Delta)=-\sum_{i=0}^{\Delta-1} \frac{1}{x+i}$ for $\Delta\in\bbN$. 
We cannot invert $f_\Delta(\theta)$ since it involves solving the root of a high-degree polynomial.}


\item[Type 3.] This type of families has discrete support $\calX$ that requires an exponential time to compute the log-normalizer.
For example, consider the Ising models~\cite{MRF-1986,Ising-2000,Ising-2018}:
Let $G = (V, E)$ be an undirected graph of $|V|$ nodes and $|E|$ edges. Each node $v\in V$ is associated with a
binary random variable $x_v\in\{0, 1\}$. 
The probability of an Ising model is defined as follows:
\begin{equation}
p(x;\theta)=\exp\left( 
\sum_{v\in V} \theta_v x_v +   \sum_{(v,w)\in E} \theta_{vw} x_vx_w -F(\theta)
 \right).
\end{equation}
The vector $t(x)=(\ldots,x_v,\ldots,x_{vw},\ldots)$ of sufficient statistics is $D$-dimensional with $D=|V|+|E|$.
The log-normalizer is:
\begin{equation}
F(\theta)=\log\left(
\sum_{(x_v)_v\in\{0,1\}^{|V|}} \left( \exp \sum_{v\in V} \theta_vx_v +   \sum_{(v,w)\in E} \theta_{vw}x_vx_w \right)
\right).
\end{equation}
It requires to sum up $2^{|V|}$ terms.

\item[Type 4.] This type of families has provably the Bregman generator that is not available in closed-form.
For example, this is the case of
the {\em Polynomial Exponential Family}~\cite{Cobb-1983,PEF-2016} (PEF) that are helpful to model a multimodal distribution (instead of using a statistical mixture).
Consider the following vector of sufficient statistics $t(x)=(x,x^2,\ldots, x^D)$ for defining an exponential family:

\begin{equation}
\calE_{t(x),\mu}=\left\{ p(x;\theta)=\exp\left(\sum_{i=1}^D \theta_i x^i-F(\theta)\right) \st \theta\in\Theta\right\}.
\end{equation}
(Beware that here, $x^i=\mathrm{Pow}(x,i)\eqdef \underbrace{x\times\ldots\times x}_{\mbox{$i$ times}}$ denotes the $i$-th power of $x$ (monomial of degree $i$), and not a contravariant coefficient of a vector $x$.)
 
In general, the definite integral of the cumulant function (the Exponential Family Bregman Generator, EFBG) of Eq.~\ref{eq:cf} does not admit a closed form, but is analytic.
For example, choosing $t(x)=x^8$, we have:
\begin{equation}
F(\theta)=\log \int_{-\infty}^\infty \exp(\theta x^8) \dx = \log 2+\log \Gamma(9/8)-\frac{1}{8}\log(-\theta),
\end{equation}
for $\theta<0$.
But $\int_{-\infty}^\infty \exp(-x^8-x^4-x^2) \dx \simeq 1.295$ is not available in closed form.

\item[Type 5.] This last category is even more challenging from a computational point of view because of log-sum terms.
For example, the {\em mixture family}.
As already stated, the negative Shannon entropy (i.e., the Mixture Family Bregman Generator, MFBG) is not available in closed form for statistical mixture models~\cite{wmixture-2017}.
It is in fact even worse, as the Shannon entropy of mixtures is not analytic~\cite{KLnotanalytic-2004}.

\end{description}

This paper considers approximating the computationally untractable generators of statistical exponential/mixture families (type $4$ and type $5$) using stochastic Monte Carlo approximations.

In~\cite{Critchley-CIG}, Critchley et al. take a different approach of the computational tractability by discretizing the support $\calX$ into a finite number of bins, and considering the corresponding discrete distribution. However, this approach does not scale well with the dimension of the support. Our Monte Carlo Information Geometry scales to arbitrary high dimensions because it relies on the fact that the Monte Carlo stochastic estimator is independent of the dimension~\cite{MC-2014}.

\subsection{Paper organization}
In \S\ref{sec:MCIGMF}, we consider the MCIG structure of mixture families: 
Namely, \S\ref{sec:MCIGMM1D} considers first the uni-order families just to illustrate the basic principle. 
It is followed by the general case in \S\ref{sec:MCIGMMDD}.
Similarly, \S\ref{sec:MCIGEF} handles the exponential family case by first explaining the uni-order case in \S\ref{sec:MCIGEF1D} before tackling the general case in \S\ref{sec:MCIGEFDD}.
\S\ref{sec:MCIGClust} presents an application of the computationally-friendly MCIG structures for clustering distributions in dually flat statistical  mixture manifolds.
Finally, we conclude and discuss several perspectives in \S\ref{sec:concl}.

\section{Monte Carlo Information Geometry of Mixture Families}\label{sec:MCIGMF}

Recall the definition  of a statistical mixture model (Definition~\ref{def:mf}):
Given a set of $k$ prescribed statistical distributions $p_0(x),\ldots, p_{k-1}(x)$, all sharing the same support $\calX$, a
 {\em mixture family} $\calM$ of order $D=k-1$ consists in all {\em strictly convex combinations} of the $p_i(x)$'s~\cite{wmixture-2017}:

\begin{equation}
\calM \eqdef  \left\{ m(x;\eta)= \sum_{i=1}^{k-1} \eta_ip_i(x)+ (1-\sum_{i=1}^{k-1} \eta_i)p_0(x) \st \eta_i>0, \sum_{i=1}^{k-1} \eta_i<1 \right\}.
\end{equation}

The differential-geometric structure of $\calM$ is well studied in information geometry~\cite{IG-2014,IG-2016} (although much less than for the exponential families), where it is known that:
\begin{equation}
\KL(m(x;\eta):m(x;\eta')) = B_G(\eta:\eta'),
\end{equation}
for the Bregman generator being the Shannon negative entropy (MFBG):
\begin{equation} 
G(\eta)=-h(m(x;\eta))=\int_{x\in\calX} m(x;\eta)\log m(x;\eta) \dmu(x).
\end{equation}
 
The negative entropy $G(\eta)=\int_{x\in\calX} m(x;\eta)\log m(x;\eta)\dmu(x)$ is a smooth and strictly convex function which induces a dually flat structure with Legendre convex conjugate:
\begin{equation}
F(\theta)=G^*(\theta)=-\int_{x\in\calX} p_0(x)\log m(x;\eta)\dmu(x)=h^\times(p_0(x):m(x;\eta)),
\end{equation} 
interpretable as the cross-entropy of $p_0(x)$ with the mixture $m(x;\eta)$~\cite{wmixture-2017}.

Notice that the component distributions may be heterogeneous like $p_0(x)$ being a fixed Cauchy distribution, $p_1(x)$ being a fixed Gaussian distribution, $p_2(x)$ a Laplace distribution, etc.
Except for the case of the finite categorical distributions (that are interpretable both as either a mixture family and an exponential family, see~\cite{IG-2016}), $G(\eta)$  provably does not admit a closed form~\cite{KLnotanalytic-2004} (i.e., meaning that the definite integral of Eq.~\ref{eq:negent} does not admit a simple formula using common standard functions). 
Thus the    dually-flat geometry $(\calM,G)$ is a theoretical construction that cannot be explicitly used by Bregman algorithms.

One way to tackle the lack of closed form of Eq.~\ref{eq:negent}, is to approximate definite integrals whenever they are used by using Monte Carlo stochastic integration.
However, this is computationally very expensive, and, even worse, it cannot guarantee that the overall computation is consistent.

Let us briefly explain the meaning of {\em consistency}:
We can estimate the KL between two distributions $p$ and $q$ by drawing $m$ variates $x_1,\ldots,x_m \sim p(x)$, and use the 
the following MC KL estimator:
\begin{equation}
\widehat{\KL}_m(p:q) \eqdef \frac{1}{m} \sum_{i=1}^m    \log \frac{p(x_i)}{q(x_i)}.
\end{equation}

Now, suppose we have $\KL(p:q)\leq \KL(q:r)$, then their MC estimates may not satisfy $\widehat{\KL}_m(p:q)<\widehat{\KL}_m(q:r)$ (since each time we evaluate a $\widehat{\KL}_m$ we draw different variates).
Thus when running a KL/Bregman algorithm, the more MC stochastic  approximations of integrals are performed in the algorithm, the less likely is the output consistent.
For example, consider computing the Bregman Voronoi diagram~\cite{BVD-2007} of a set of $n$ mixtures belonging to a mixture family manifold (say, with $D=2$) using the algorithm explained in~\cite{BVD-2007}: 
Since we use for each BD calculation or predicate evaluation relying on $F$ or $F^*$ stochastic Monte Carlo integral approximations, this MC algorithm may likely not deliver a proper  combinatorial structure of the Voronoi diagram as its output: The Voronoi structure is likely to be inconsistent.

Let us now show how Monte Carlo Information Geometry (MCIG) approximates this computationally untractable $(\calM,G)$ geometric structure by defining a consistent and computationally-friendly dually-flat information geometry
 $(\calM,\tildeG_\calS)$ for a finite identically and independently distributed (iid) random sample $\calS$ of prescribed size $m$.

\subsection{MCIG of Order-$1$ Mixture Family}\label{sec:MCIGMM1D}
In order to highlight the principle of MCIGs, let us first consider a mixture family of order $D=1$.
That is, we consider a set of  mixtures of $k=2$ components with density: 
\begin{equation}
m(x;\eta)=\eta p_1(x)+(1-\eta)p_0(x)=p_0(x)+\eta(p_1(x)-p_0(x)),
\end{equation} 
with parameter $\eta$ ranging in $(0,1)$.
  The two prescribed component densities  $p_0(x)$ and $p_1(x)$ (with respect to a base measure $\mu$, say the Lebesgue measure) are defined on a common support $\calX$.
Densities $p_0(x)$ and $p_1(x)$ are assumed to be linearly independent~\cite{IG-2014}.

Figure~\ref{fig:mmexample} displays an example of uni-order mixture family with heterogeneous components: 
$p_0(x)$ is chosen as a Gaussian distribution while $p_1(x)$ is taken as a Laplace distribution.
A mixture $m(x;\eta)$ of $\calM$ is visualized as a point $P$ (here, one-dimensional) with $\eta(P)=\eta$.

Let $\calS=\{x_1,\ldots, x_m\}$ denote a iid sample from a fixed {\em proposal distribution} $q(x)$ (defined over the same support $\calX$, and independent of $\eta$). We approximate the  Bregman generator $G(\eta)$ using Monte Carlo stochastic integration with importance sampling as follows:

\begin{equation}\label{eq:mfag1D}
G(\eta) \simeq  \tildeG_\calS(\eta) \eqdef \frac{1}{m} \sum_{i=1}^m  \frac{1}{q(x_i)} m(x_i;\eta)\log m(x_i;\eta).
\end{equation}

\begin{figure}%
\centering
\includegraphics[width=0.85\columnwidth]{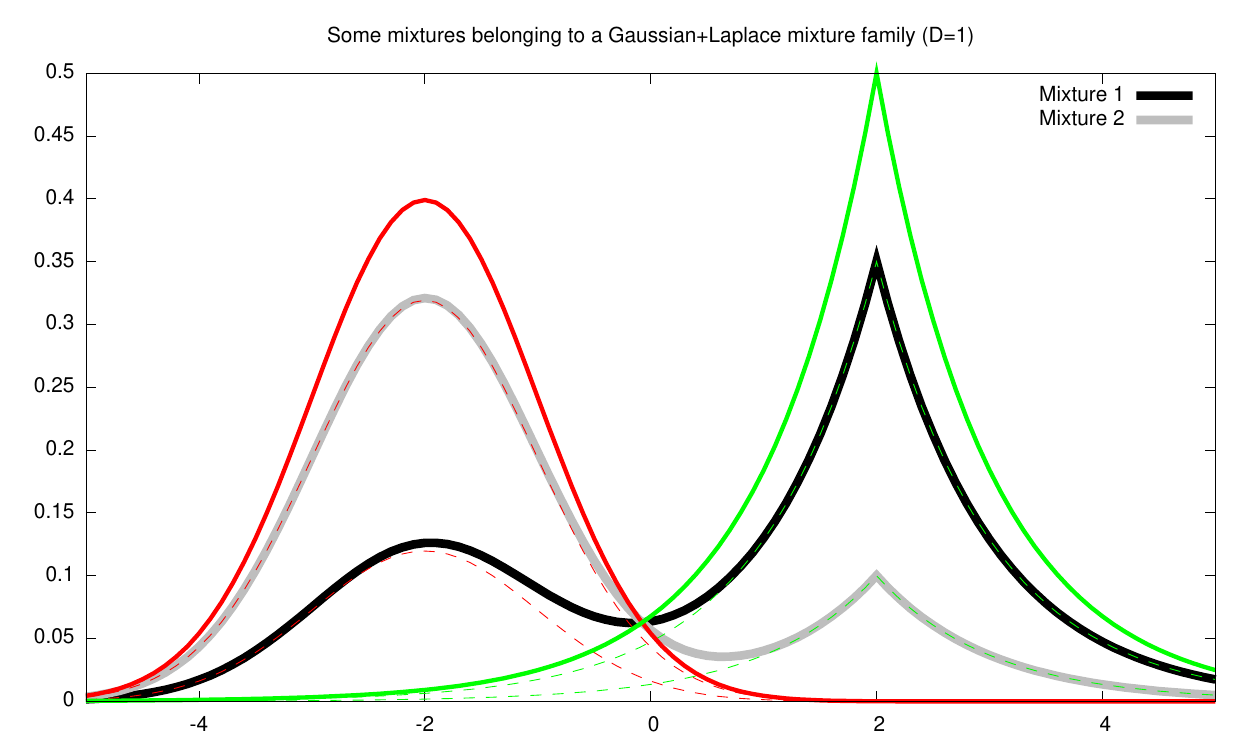}%

\caption{Example of a mixture family of order $D=1$ ($k=2$): $p_0(x)\sim \mathrm{Gaussian}(-2,1)$ (red) and $p_1(x)\sim \mathrm{Laplace}(2,1)$ (green). 
The two mixtures are $m_1(x)=m(x;\eta_1)$ (black) with $\eta_1=0.7$ and $m_2(x)=m(x;\eta_2)$ (grey) with $\eta_2=0.2$. Weighted component distributions  are displayed in dashed.
 }%
\label{fig:mmexample}%
\end{figure}

Let us prove that the Monte Carlo function $\tildeG_\calS(\eta)$ is a proper Bregman generator.
That is, that  $\tildeG_\calS(\eta)$ is strictly convex and twice continuously differentiable (Definition~\ref{def:bg}).

Write for short $m_x(\eta) \eqdef m(x;\eta)$ so that $G(\eta)=\int_{x\in\calX} m_x(\eta)\log m_x(\eta)\dmu(x)$ is 
approximated by  $\frac{1}{m} \sum_{i=1}^m  \frac{1}{q(x_i)} m_{x_i}(\eta)\log m_{x_i}(\eta)$.
Since $\frac{1}{m} \frac{1}{q(x_i)}>0$, it suffices to prove that the  basic function $g_x(\eta)=m_{x}(\eta)\log m_{x}(\eta)$ is strictly convex wrt parameter $\eta$.
Then we shall conclude that $\tildeG_\calS(\eta)$ is strictly convex because it is the  finite positively weighted sum of strictly convex functions.

Let us write the first and second derivatives of $g_x(\eta)$ as follows:
\begin{eqnarray}
g_x(\eta)' &=&m_{x}(\eta)'(\log m_{x}(\eta)+1),\\
g_x(\eta)'' &=& m_{x}(\eta)''(\log m_{x}(\eta)+1)+\frac{(m_{x}(\eta)')^2}{m_{x}(\eta)}.
\end{eqnarray}

Since $m'_x(\eta)=p_1(x)-p_0(x)$ and $m''_x(\eta)=0$, we get:

\begin{equation}
g_x(\eta)''= \frac{(p_1(x)-p_0(x))^2}{m_{x}(\eta)}.
\end{equation}
 
Thus it follows that:
\begin{equation}\label{eq:mmg1D}
\tildeG_\calS''(\eta) = \frac{1}{m} \sum_{i=1}^m  \frac{1}{q(x_i)} \frac{(p_1(x_i)-p_0(x_i))^2}{m(x_i;\eta)} \geq 0.
\end{equation}
It is strictly convex provided that there exists at least one $x_i$ such that $p_1(x_i)\not= p_0(x_i)$.

Let $\calD\subset\calX$ denote the degenerate set $\calD=\{x\in\calX \st p_1(x) = p_0(x) \}$.
For example, if $p_0(x)$ and $p_1(x)$ are two distinct univariate normal distributions, then $|\calD|=2$ (roots of a quadratic equation), and 
\begin{equation}
\mu_q(\calD) \eqdef \int_{x\in\calX} 1_{[p_0(x) = p_1(x)]} q(x)\dmu(x) = 0.
\end{equation}

\begin{assumption}[AMF1D]
We assume that $p_0(x)$ and $p_1(x)$ are linearly independent (non-singular statistical model, see~\cite{IG-2014}), and that $\mu_q(\calD)=0$.
\end{assumption}

\begin{lemma}[Monte Carlo Mixture Family Function is a Bregman generator]
The Monte Carlo Mixture Family Function (MCMFF) $\tildeF_\calS(\theta)$ is a Bregman generator almost surely.
\end{lemma}

\begin{proof}
When there exists a sample $x\in\calS$ with two distinct densities $p_0(x)$ and $p_1(x)$, we have $(p_1(x_i)-p_0(x_i))^2>0$ and
therefore $\tildeG_\calS''(\eta)>0$. 
The probability to get a degenerate sample is almost zero.
\end{proof}

To recap, the MCMFF of the MCIG of uni-order family has the following characteristics:
\vskip 0.5cm
\noindent \fbox{\vtop{\underline{Monte Carlo Mixture Family Generator 1D}: 
\begin{eqnarray} 
\tildeG_\calS(\eta) &=& \frac{1}{m} \sum_{i=1}^m  \frac{1}{q(x_i)} m(x_i;\eta)\log m(x_i;\eta),\\
\tildeG_\calS'(\eta) &=&  \theta = \frac{1}{m} \sum_{i=1}^m  \frac{1}{q(x_i)} (p_1(x_i)-p_0(x_i)) (1+\log m(x_i;\eta)),\\
\tildeG_\calS''(\eta) &=& \frac{1}{m} \sum_{i=1}^m  \frac{1}{q(x_i)} \frac{(p_1(x_i)-p_0(x_i))^2}{m(x_i;\eta)}.
\end{eqnarray}
}}
\vskip 0.5cm

Note that $(G^*)'$ and $G^*$ may be calculated numerically but not in closed-form.
We may also MC approximate $\nabla G^*$ since $\theta=(h^\times(p_0:m)-h^\times(p_i:m))_i$.

Thus we change from type 5 to type 2 the computational tractability of mixtures by adopting the MCIG approximation. 

Figure~\ref{fig:mmgeneratorseq} displays a series of  Bregman mixture family MC generators for a mixture family for different values of $|\calS|=m$.

\begin{figure}%
\centering

\includegraphics[width=0.8\columnwidth]{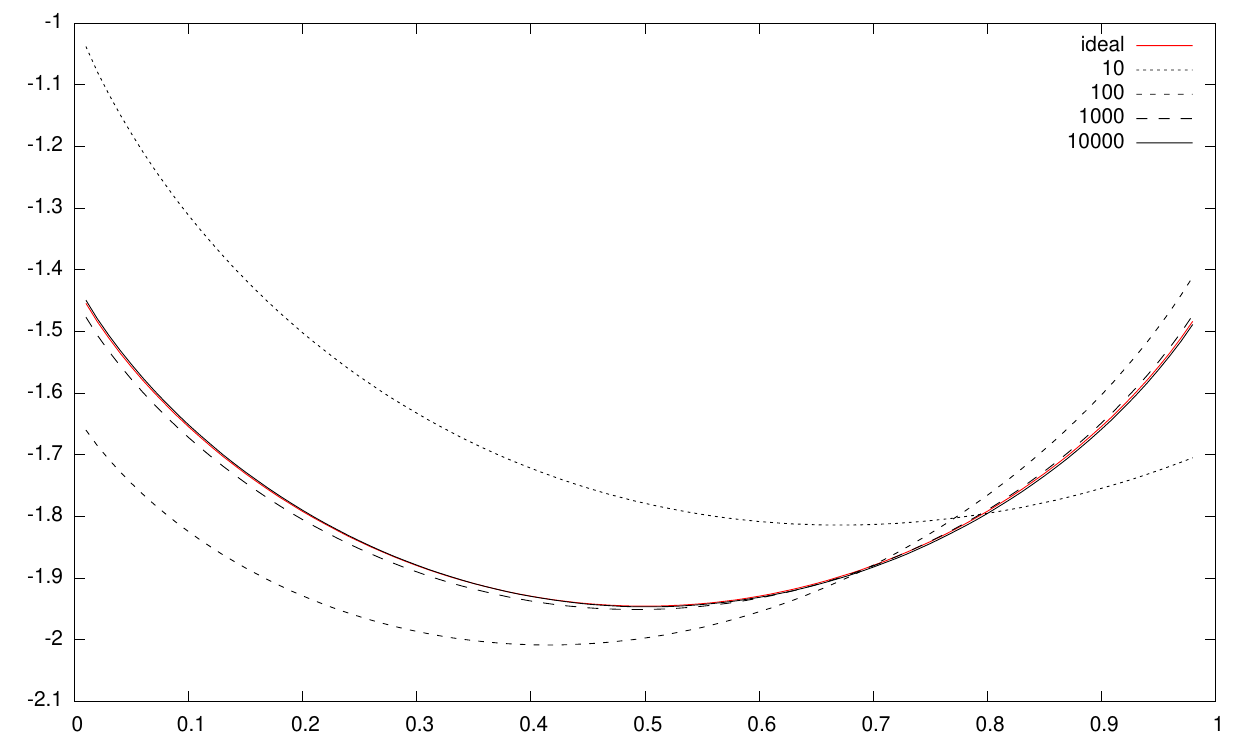}

\caption{A series $G_\calS(\eta)$ of Bregman Monte Carlo Mixture Family generators (for $m=|\calS|\in\{10,100,1000,10000\}$) approximating the untractable ideal negentropy generator $G(\eta)=-h(m(x;\eta))$ (red) of a  mixture family with prescribed Gaussian distributions $m(x;\eta)=(1-\eta)p(x;0,3)+\eta p(x;2,1)$ for the proposal distribution $q(x)=m(x;\frac{1}{2})$.
 }%
\label{fig:mmgeneratorseq}%
\end{figure}

%
%

As we increase the sample size of $\calS$, the MCMFF Bregman generator tends to the ideal mixture family Bregman generator.

\begin{theorem}[Consistency of MCIG]
Almost surely, $\lim_{m\rightarrow\infty} (\calM,\tildeG_{\calS}) = (\calM,G)$ when $\mu_q(\calD)=0$.
\end{theorem}
\begin{proof}
It suffices to prove that $\lim_{m\rightarrow\infty} \tildeG_{\calS}(\eta) = G(\eta)$.
The general theory of Monte Carlo stochastic integration yields a consistent estimator provided that the following variance is bounded
\begin{equation}
\Var_q\left[\frac{m(x;\eta)\log m(x;\eta)}{q(x)}\right]<\infty.
\end{equation}
 
For example, when $m(x;\eta)$ is a mixture of prescribed isotropic gaussians (say, from a KDE), and $q(x)$ is also an isotropic Gaussian, the variance is bounded.
Note that $q$ is the proposal density wrt the base measure $\mu$.
\end{proof}

In practice, the proposal distribution $q(x)$ can be chosen as the uniform  mixture of the fixed component distributions: 
\begin{equation}
q(x)=\frac{1}{m} \sum_{i=0}^D p_i(x).
\end{equation}

Notice that the Monte Carlo Mixture Family Function is a random variable (rv) estimator itself by considering a vector of iid variables instead of a sample variate: $\hatG_m(\eta)$.
Figure~\ref{fig:MCMFBG1D} displays five realizations of the random variable $\hatG_m(\eta)$ for $m=10$.

\begin{figure}%
\centering

\includegraphics[width=0.8\columnwidth]{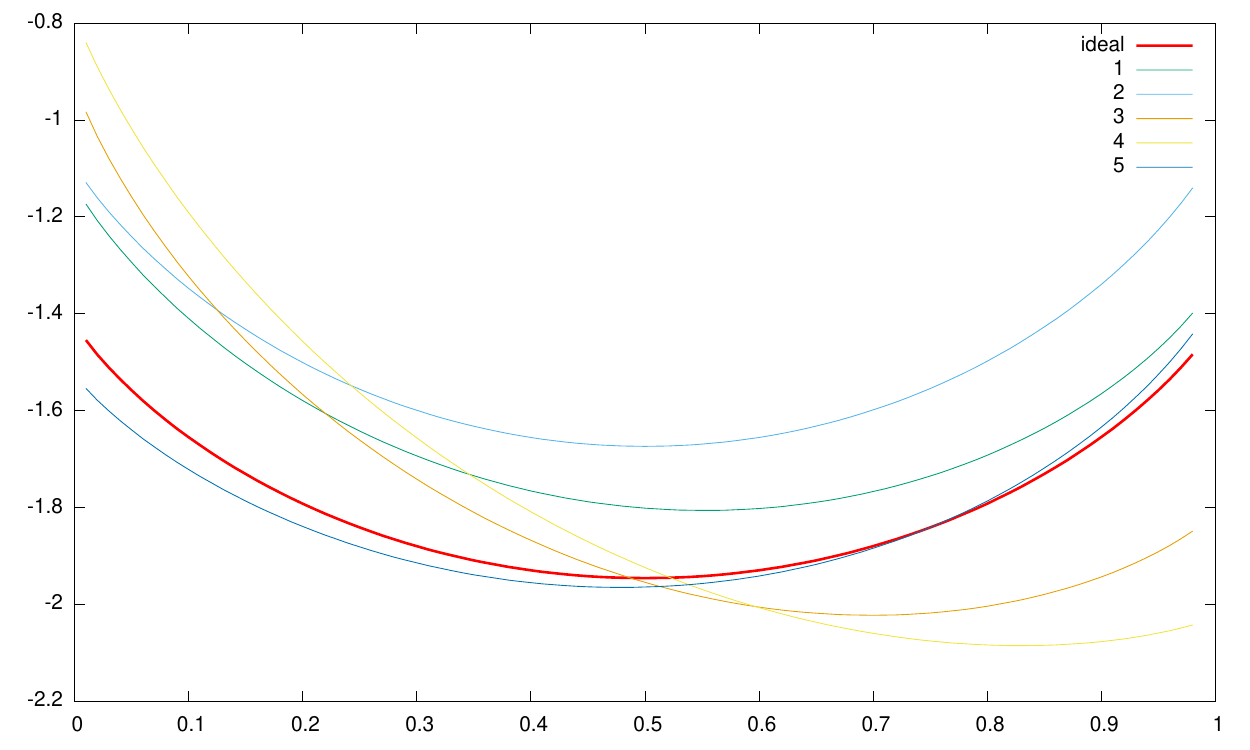}%

\caption{The Monte Carlo Mixture Family Generator $\hatG_{10}$ (MCMFG) considered as a random variable: Here, we show five realizations (i.e., $\calS_1,\ldots,\calS_5$) of the randomized generator for $m=5$. The ideal generator is plot in thick red.
 }%
\label{fig:MCMFBG1D}%
\end{figure}

\subsection{General $D$-order mixture case}\label{sec:MCIGMMDD}

Here, we consider statistical mixtures with $k=D+1>2$ prescribed distributions $p_0(x),\ldots, p_D(x)$.
The component distributions are linearly independent so that they define a non-singular statistical model~\cite{IG-2014}.

We further strengthen conditions on the prescribed distributions as follows:

\begin{assumption}[AMF]
We assume that the linearly independent prescribed distributions further satisfy:
\begin{equation}
  \mathrm{sup}_{B \in \mathcal{B}}\left\{ \mu_q(B) : \exists \lambda
    \neq (0), \sum_{i \neq j} \lambda_i \left( \left.p_i\right|_B - \left.p_j\right|_B
    \right) = 0  \right\} = 0, \quad \forall j,
\end{equation}
where the supremum is over all subsets $B$ of the $\sigma$-algebra
$\mathcal{B}$ of the probability space with support $\mathcal{X}$ and
measure $\mu$,  with $\left.p_i\right|_B$ denoting the restriction of
$p_i$ to subset $B$. In other words, we impose that
the components $(p_i)_i$ still constitute an
affinely independent family when restricted to any subset of positive
measure.
\end{assumption}
For example, Figure~\ref{fig:mmexample2D} displays two mixture distributions belonging to a 2D mixture family with Gaussian, Laplace and Cauchy component distributions.

Recall that the mixture family Monte Carlo generator is:
\begin{equation}
\tildeG_\calS(\eta) = \frac{1}{m} \sum_{i=1}^m  \frac{1}{q(x_i)} m(x_i;\eta)\log m(x_i;\eta).
\end{equation}

In order to prove that $G$ is strictly convex, we shall prove that
$\nabla^2 \tildeG_\calS(\eta)\succ 0$ almost surely.
It suffices to consider the basic Hessian matrix $\nabla^2 g_x=(\partial^i\partial^j g_x(\eta))_{ij}$ of $g_x(\eta)=m_x(\eta)\log m_x(\eta)$.
We have the partial first derivatives:
\begin{equation}
\partial^i g_x(\eta)=(p_i(x)-p_0(x))(1+\log m(x;\eta)),
\end{equation}
and the partial second derivatives:
\begin{equation}
\partial^i \partial^j g_x(\eta)= \frac{(p_i(x)-p_0(x))(p_j(x)-p_0(x))}{m(x;\eta)},
\end{equation}
so that
\begin{equation}\label{eq:MMGDD}
\partial^i \partial^j \tildeG_\calS(\eta) =   \frac{1}{m} \sum_{l=1}^m \frac{1}{q(x_l)} \frac{(p_i(x_l)-p_0(x_l))(p_j(x_l)-p_0(x_l))}{m(x_l;\eta)}.
\end{equation}

\begin{theorem}[Monte Carlo Mixture Family Function is a Bregman generator]\label{thm:MCMF}
The Monte Carlo multivariate function $\tildeG_\calS(\eta)$ is always convex and twice continuously differentiable, and strictly convex  almost surely.
\end{theorem}

\begin{proof}
Consider the $D$-dimensional vector:
\begin{equation}
v_l = \left[
\begin{array}{c}  
\frac{p_1(x_l)-p_0(x_l)}{\sqrt{q(x_l) m(x_l;\eta)}}\\
\vdots\\
\frac{p_D(x_l)-p_0(x_l)}{\sqrt{q(x_l) m(x_l;\eta)}}\\
\end{array}
\right].
\end{equation}

Then we rewrite the Monte Carlo generator $\tildeG_\calS(\eta)$ as: 
\begin{equation}
\partial^i \partial^j \tildeG_\calS(\eta) =   \frac{1}{m} \sum_{l=1}^m v_l v_l^\top.
\end{equation}

Since $v_l v_l^\top$ is always a symmetric positive semidefinite matrix of rank one,
we conclude that $\tildeG_\calS(\eta) $  is a symmetric positive semidefinite matrix when $m<D$ (rank deficient) and a symmetric positive definite matrix (full rank)  almost surely when $m\geq D$.
\end{proof}

\begin{figure}%
\centering
\includegraphics[width=0.85\columnwidth]{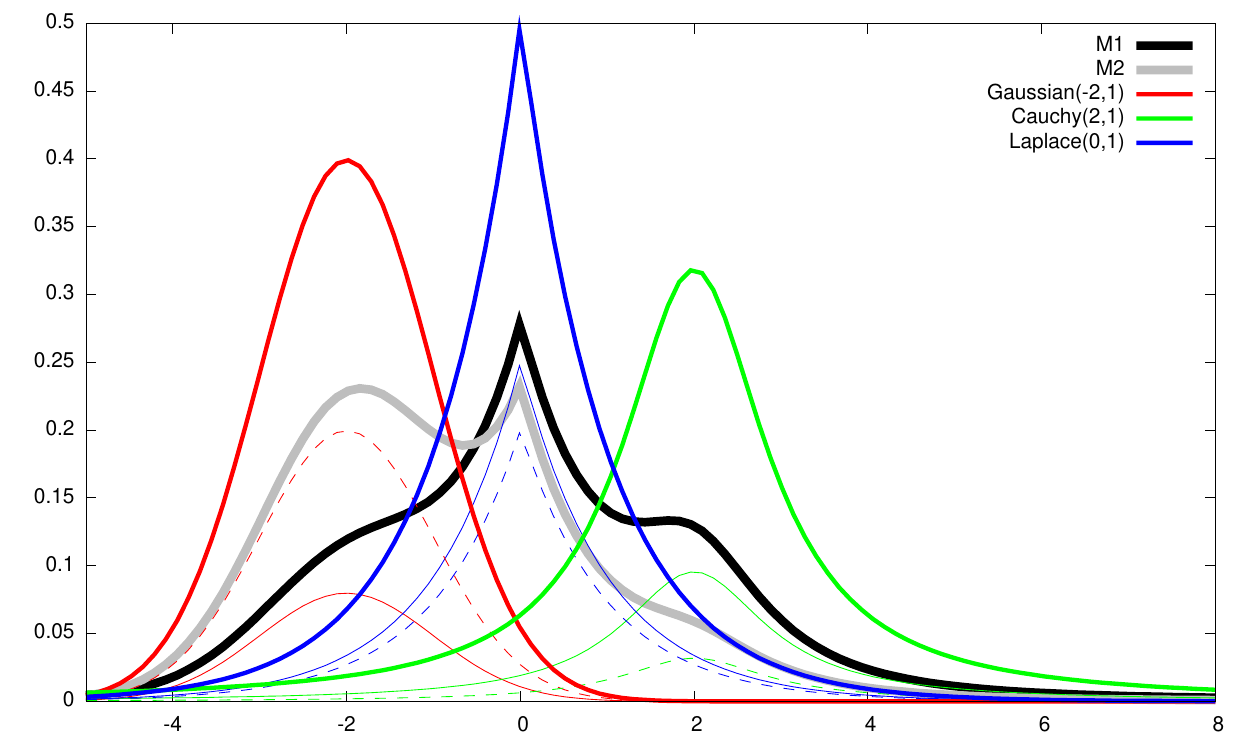}%

\caption{Example of a mixture family of order $D=2$ ($k=3$): $p_0(x)\sim \mathrm{Gaussian}(-2,1)$ (red), $p_1(x)\sim \mathrm{Laplace}(0,1)$ (blue) and
 $p_2(x)\sim \mathrm{Cauchy}(2,1)$ (green). 
The two mixtures are $m_1(x)=m(x;\eta_1)$ (black) with $\eta_1=(0.3,0.5)$ and $m_2(x)=m(x;\eta)$ (gray) with $\eta=(0.1,0.4)$.
 }%
\label{fig:mmexample2D}%
\end{figure}

\section{Monte Carlo Information Geometry of Exponential Families}\label{sec:MCIGEF}

We follow the same outline as for mixture familes: \S\ref{sec:MCIGEF1D} first describes the univariate case.
It is then followed by the general multivariate case in  \ref{sec:MCIGEF1D}.

\subsection{MCIG of Order-$1$ Exponential Family}\label{sec:MCIGEF1D}
We consider the order-$1$ exponential family of parametric densities with respect to a base measure $\mu$:
\begin{equation}
\calE \eqdef \left\{p(x;\theta)=\exp(t(x)\theta-F(\theta)+k(x)) \st \theta\in\Theta \right\},
\end{equation}
where $\Theta$ is the natural parameter space, such that the log-normalizer/cumulant function~\cite{IG-2016} is
\begin{equation}
F(\theta)=\log\left(\int \exp(t(x)\theta+k(x))\dmu(x)\right).
\end{equation}

The sufficient statistic function $t(x)$ and $1$ are linearly independent~\cite{IG-2014}.

We perform Monte Carlo stochastic integration by sampling a set $\calS=\{x_1,\ldots, x_m\}$ of $m$ iid  variates from a proposal distribution $q(x)$ to get:
\begin{equation}\label{eq:fdagger}
F(\theta) \simeq \tildeF_\calS^\dagger(\theta) \eqdef \log \left( \frac{1}{m} \sum_{i=1}^m \frac{1}{q(x_i)} \exp(t(x_i)\theta+k(x_i)) \right).
\end{equation}

Without loss of generality, assume that $x_1$ is the element that minimizes the sufficient statistic $t(x)$ among the elements of $\calS$, so that
$a_i=t(x_i)-t(x_1)\geq 0$ for all $x_i\in\calS$.

Let us factorize $\frac{1}{q(x_1)} \exp(t(x_1)\theta+k(x_1))$ in Eq.~\ref{eq:fdagger} and remove an affine term from the generator $\tildeF_\calS(\theta)$ to get the equivalent generator (see Property~\ref{prop:bdaffine}):
\begin{eqnarray}
\tildeF_\calS^\dagger(\theta)  &\equiv& \tildeF_\calS(\theta),\\
\tildeF_\calS(\theta) &=&\log \left(1+ \sum_{i=2}^m \exp((t(x_i)-t(x_1))\theta+k(x_i)-k(x_1)-\log q(x_i)+\log q(x_1)) \right),\\
&=&\log \left(1+ \sum_{i=2}^m \exp(a_i\theta+b_i)\right),\\
&\eqdef& \lse_0^+(a_2\theta+b_2, \ldots, a_m\theta+b_m),\label{eq:gexpf}
\end{eqnarray}
with $a_2,\ldots, a_m>0$ and $b_i=k(x_i)-k(x_1)-\log q(x_i)+\log q(x_1)$.
Function $\lse_0^+(x_1,\ldots,x_m)=\lse(0,x_1,\ldots,x_m)$ is the log-sum-exp function~\cite{lse-2016,lse-2017} $\lse(x_1,\ldots, x_m)=\log \sum_{i=1}^n \exp(x_i)$ with an additional argument set to zero.

Let us notice that the $\lse_0^+$ function is always {\em strictly convex} while the $\lse$ function is only convex\footnote{
Function lse can be interpreted as a vector function, and is  $C^2$, convex but not strictly convex on $\bbR^m$. 
For example, lse is affine on lines   since $\lse(x+\lambda 1)=\lse(x)+\lambda$ (or equivalently $\lse(x_1,\ldots, x_m)=\lambda + \lse(x_1-\lambda,\ldots,x_m-\lambda)$). It is affine only on lines passing through the origin.}~\cite{Boyd-2004}, p. 74.
Figure~\ref{fig:lselse0p} displays the graph plots of the lse and $\lse_0^+$ functions.
Let us clarify this point with a usual exponential family: The binomial family.
The binomial distribution is a categorical distribution with $D=1$ (and $2$ bins). 
We have $F(\theta)=\log(1+\exp(\theta))=\lse(0,\theta)\eqdef\lse_0^+(\theta)$.
We check the strict convexity of $F(\theta)$: $F'(\theta)=\frac{e^\theta}{1+e^\theta}$ and $F''(\theta)=\frac{e^\theta}{(1+e^\theta)^2}>0$.

We write for short $\lse_0^+(x)=\lse_0^+(x_1,\ldots, x_d)$ for a $d$-dimensional vector $x$.

\begin{theorem}[$\lse_0^+$ is a Bregman generator]\label{thm:lse0pBF}
Multivariate function $\lse_0^+(x)$ is a Bregman generator.
\end{theorem} 

Proof is deferred to Appendix~\ref{sec:proof-lse0p}.

\begin{lemma}[Univariate Monte Carlo  Exponential Family Function is a Bregman generator]
Almost surely, the univariate function $\tildeF_\calS(\theta)$ is a Bregman generator.
\end{lemma}

\begin{proof}
The first derivative is:
\begin{equation}
\eta= \tildeF_\calS'(\theta)=\frac{ \sum_{i=2}^m a_i\exp(a_i\theta+b_i)}{1+ \sum_{i=2}^m \exp(a_i\theta+b_i)} \geq 0,
\end{equation}
and is strictly greater than $0$ when there exists at least two elements with distinct sufficient statistics (i.e., $t(x_i)\not= t(x_j)$) so that at least one $a_i>0$.

The second derivative is:
\begin{eqnarray}
\tildeF_\calS''(\theta)=\frac{ \left(\sum_{i=2}^m a_i^2\exp(a_i\theta+b_i)\right)  \left(1+ \sum_{i=2}^m \exp(a_i\theta+b_i)\right)
- \left(\sum_{i=2}^m a_i\exp(a_i\theta+b_i)\right)^2
 }{(1+ \sum_{i=2}^m \exp(a_i\theta+b_i))^2}  \defeq \frac{\Num}{\Den} 
\end{eqnarray}

For each value of $\theta\in\Theta$, we shall prove that $\tildeF_\calS''(\theta)>0$.
Let $c_i=c_i(\theta)=\exp(a_i\theta+b_i)>0$ for short ($\theta$ being fixed, we omit it in the $c_i$ notation in the  calculus derivation).
Consider the numerator $\Num$ since the denominator $\Den$ is a non-zero square, hence strictly positive.
We have:
\begin{eqnarray}
\Num &>& \left(\sum_{i=2}^m a_i^2 c_i\right)  \left(\sum_{i=2}^m c_i\right) - \left(\sum_{i=2}^m a_i c_i\right)^2,\\
\Num &>& \sum_{ij} a_i^2c_ic_j - \sum_i a_i^2c_i^2 -2\sum_{i<j} a_ia_jc_ic_j,\\
\Num &>& \sum_{i=j} a_i^2c_i^2 + \sum_{i\not =j} a_i^2c_ic_j - \sum_i a_i^2c_i^2 -2\sum_{i<j} a_ia_jc_ic_j,\\
\Num &>& \sum_{i<j} a_i^2c_ic_j+\sum_{i>j} a_i^2c_ic_j-2\sum_{i<j} a_ia_jc_ic_j,\\
\Num &>& \sum_{i<j} a_i^2c_ic_j+\sum_{i<j} a_j^2c_ic_j -2\sum_{i<j} a_ia_jc_ic_j,\\
\Num  &>&  \sum_{i<j}   (a_i^2+a_j^2-2a_ia_j) c_ic_j,\\
\Num &>& \sum_{i<j} (a_i-a_j)^2 c_ic_j >0.
\end{eqnarray}

Therefore the numerator is strictly positive if at least two $a_i$'s are distinct.
\end{proof}

Thus we add the following assumption:

\begin{assumption}[AEF1D]
For all $y\in\dom(t)$, $E_q[1_{t(x)=y}]=0$.
\end{assumption}

\begin{figure}
\centering
\includegraphics[width=0.7\columnwidth]{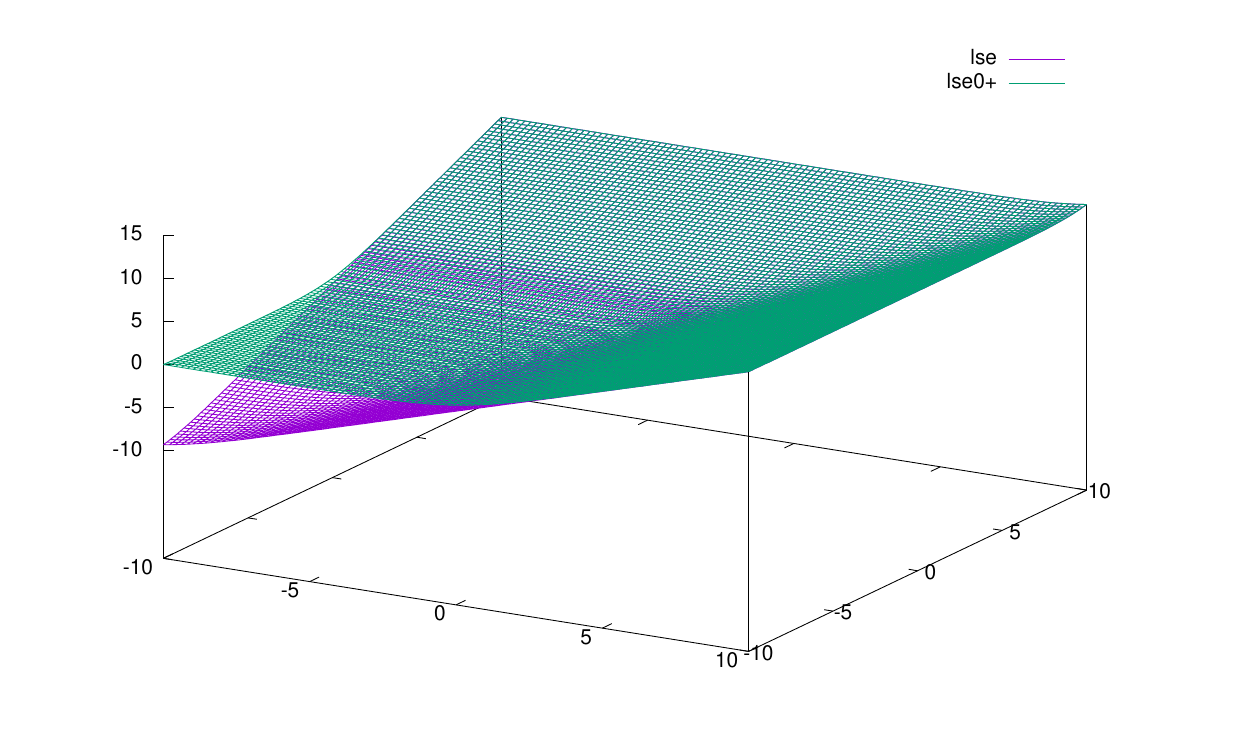}

\caption{Graph plots of the $\lse$ and $\lse_0^+$ functions: The $\lse$ function (violet) is only convex while the $\lse_0^+$ function (green) is always guaranteed to be strictly convex.\label{fig:lselse0p}}
\end{figure}

To recap, the MCEFF of the MCIG of uni-order family has the following characteristics:
\vskip 0.5cm
\noindent \fbox{\vtop{\underline{Monte Carlo Mixture Family Generator 1D}: 
\begin{eqnarray} 
\tildeF_\calS(\theta) &=& \lse_0^+(a_2\theta+b_2, \ldots, a_m\theta+b_m),\\
&& a_i=t(x_i)-t(x_1),\\
&& b_i=k(x_i)-k(x_1)-\log q(x_i)+\log q(x_1),\\
\tildeF_\calS'(\theta) &=&\frac{ \sum_{i=2}^m a_i\exp(a_i\theta+b_i)}{1+ \sum_{i=2}^m \exp(a_i\theta+b_i)} \defeq \eta,\\
\tildeF_\calS''(\theta) &=&\frac{ \left(\sum_{i=2}^m a_i^2\exp(a_i\theta+b_i)\right)  \left(1+ \sum_{i=2}^m \exp(a_i\theta+b_i)\right)
- \left(\sum_{i=2}^m a_i\exp(a_i\theta+b_i)\right)^2
 }{(1+ \sum_{i=2}^m \exp(a_i\theta+b_i))^2} 
\end{eqnarray}
}}
\vskip 0.5cm

\subsection{The general $D$-order case}\label{sec:MCIGEFDD}
The difference of sufficient statistics $a_i=t(x_i)-t(x_1)$ is now a vector of dimension $D$:
\begin{equation}
a_i= \left[\begin{array}{l} a_i^1\\ \vdots\\ a_i^D\end{array} \right].
\end{equation}
We replace the scalar multiplication $a_i\theta$ by an inner product $\inner{a_i}{\theta}$ in Eq.~\ref{eq:gexpf}, and let 
$c_i(\theta)=\exp(\inner{a_i}{\theta}+b_i)$ with $b_i=k(x_i)-k(x_1)-\log q(x_i)+\log q(x_1)$.
Then the Monte Carlo Exponential Family Function (MCEFF) writes concisely as:

\begin{eqnarray}
\tildeF_\calS(\theta) &=& \log \left(1+ \sum_{l=2}^m c_l(\theta)\right),\\
&\eqdef & \lse_0^+(c_2(\theta), \ldots, c_m(\theta)),\label{eq:mnvgexpf}
\end{eqnarray}

\begin{theorem}[Monte Carlo  Exponential Family Function is a Bregman Generator]\label{thm:MCEF}
Almost surely, the function $\tildeF_\calS(\theta)$ is a proper Bregman generator.
\end{theorem}

\begin{proof}
We have the gradient of first-order partial derivatives:
\begin{equation}\label{eq:gradef}
\eta_i = \partial_i \tildeF_\calS(\theta)  = \frac{ \sum_{l=2}^m a_l^i  c_l(\theta)}{1+ \sum_{l=2}^m c_l(\theta)},
\end{equation}
and the Hessian matrix of second-order partial derivatives:
\begin{equation}\label{eq:hessef}
\partial_i \partial_j\tildeF_\calS(\theta)  = \frac{
(\sum_{l=2}^m a_l^i a_l^j c_l(\theta))(1+ \sum_{l=2}^m c_l(\theta))-
 (\sum_{l=2}^m a_l^i c_l(\theta)) ( \sum_{l=2}^m a_l^j c_l(\theta))
}{(1+ \sum_{l=2}^m c_l(\theta))^2}  \defeq \frac{\Num}{\Den}.
\end{equation}

Let us prove that the Hessian matrix $\nabla^2
\tildeF_\calS(\theta)=(\partial_i
\partial_j\tildeF_\calS(\theta))_{ij}$  is always symmetric positive
semi-definite, and symmetric positive definite almost surely.

Indeed, we have:
\begin{equation}
\Num = \underbrace{\sum_k a_k^ia_k^jc_k}_{\eqdef D} + \underbrace{\sum_{k,l} a_k^ia_k^jc_kc_l-\sum_{k,l}a_k^ic_ka_l^jc_l}_{\eqdef E}.
\end{equation}

Let us rewrite $D$ as $D=C A^\top A$ with $C=\diag(c_1,\ldots, c_D)$. It follows that matrix $D$ is symmetric positive definite.
Let us prove that matrix $E$ is also SPD:

\begin{eqnarray}
E &\stackrel{\star}{=}& \sum_{k<l} a_k^ia_k^jc_k c_l + \sum_{l<k} a_k^i z_k^j c_k c_l - \sum_{k<l} a_k^i a_l^j c_k c_l - \sum_{l<k} a_k^i a_l^j c_k c_l,\label{eq:p1}\\
&\stackrel{\star\star}{=}& \sum_{k<l}\left(a_k^i a_k^j+a_l^i a_l^j-a_k^ia_l^j-a_l^ia_k^j\right)c_kc_l,\\
&=&  \sum_{k<l} (a_k^i-a_l^i) (a_k^j-a_l^j) c_kc_l.\label{eq:p2}
\end{eqnarray}

$\star$: The terms $l=k$ vanish\\
$\star\star$: After a change of variable $l\leftrightarrow k$ in the second and fourth sums of Eq.~\ref{eq:p1}.\\

Thus Eq.~\ref{eq:p2} can be rewritten as $(a_k-a_l)(a_k-a_l)^\top c_k c_l$ 
where $a_k=\left[ \begin{array}{c}a_k^1\\ \vdots\\ a_k^D\end{array}\right]$.
It follows that $E$ is a positively weighted sum of rank-1 symmetric
positive semi-definite matrices, and is therefore symmetric positive
semi-definite.

We want $y^TEy > 0$ for all $y \neq 0 \in \bbR^D$. Suppose that there exists
$y \neq 0 \in \bbR^D$ such that $y^TEy = 0$. Noting that $a^i_k - a_l^i
= t_i(x_k) - t_i(x_l)$, we can write this as
\begin{equation}\label{eq:symdefcondition}
\sum_{k<l} \left( \sum_{i}y_ic_i(t_i(x_k) - t_i(x_l)) \sum_j y_jc_j(t_j(x_k)
  - t_j(x_l)) \right) = 0,
\end{equation}
which implies
\begin{equation}
  \label{eq:symdefcond2}
  \sum_{i}y_ic_i\left(t_i(x_k) - t_i(x_l)\right) \sum_j y_jc_j\left(t_j(x_k)
  - t_j(x_l)\right) = 0, \quad \forall k < l, 
\end{equation}
since each of these terms is non negative.
In particular, we have the existence of a $y \neq 0 \in \bbR^D$ such that
\begin{equation}\label{eq:symdefcond3}
 \sum_i y_i t_i(x_k) = \sum_i y_i t_i(x_l), \quad \forall y \neq 0, \quad
  \forall k < l.
\end{equation}

\end{proof}

To get  almost surely  a Monte Carlo Bregman generator, we
introduce the following assumption:
\begin{assumption}[AEF]
  The sufficient statistics $\left(t_i\right)$ verify that for all
  $\lambda \neq 0$ and all $y \in dom(\sum_i \lambda_i t_i)$:
  \[ E_q \left[ 1_{\sum_i \lambda_i t_i(x) = y} \right] = 0. \]
\end{assumption}

\section{Application to clustering}\label{sec:MCIGClust}

In this section, we demonstrate the practical use of MCIG to cluster a set of mixtures in~\S\ref{sec:clustmix}, and consider in ~\S\ref{sec:parallel} parallel calculations/aggregations of Monte Carlo Exponential/Mixture Functions.

\subsection{Clustering mixtures on the mixture family manifold}\label{sec:clustmix}
Consider clustering a set of $n$ mixtures $m(x;\eta_1),\ldots, m(x;\eta_n)$ of the mixture family manifold.
Prior work considered clustering the mixture components (e.g., Gaussian components) to simplify mixtures by using the Bregman $k$-means~\cite{GaussianClustering-2007,GaussianClustering-2009}. This can be interpreted as a Gaussian/component quantization procedure.

Here, we address the different problem of clustering the mixtures themselves, not their components.

Since $\KL(m(x;\eta_i):m(x;\eta_j))=B_G(\eta_i:\eta_j)$ for $G(\eta)=-h(m(x;\eta))$ (Shannon information), we may approximate the KL divergence from the MC Bregman Divergence (MCBD) $\tildeG_\calS$ as follows:
\begin{eqnarray}
\KL(m(x;\eta_i):m(x;\eta_j))&=& B_G(\eta_i:\eta_j),\\
&\simeq&  B_{\tildeG_\calS}(\eta_i:\eta_j). 
\end{eqnarray}

One advantage of using a MCIG is that all divergence computations $B_{\tildeG_\calS}$ performed during the execution of a Bregman algorithm are consistent by reusing the same variates of $\calS$.
In particular, this also guarantees to always have nonnegative estimated KL divergences.

The traditional way to MC estimate the KL divergence is to consider the MC stochastic integration of the extended Kullback-Leibler divergence~\cite{Bregman-2005}:
\begin{equation}
\widehat{\eKL}_m(p:q) \eqdef \frac{1}{m} \sum_{i=1}^m \left(  \log \frac{p(x_i)}{q(x_i)}  + \frac{q(x_i)}{p(x_i)} -  1 \right),
\end{equation}
for $x_1,\ldots,x_m \sim p(x)$.
Indeed, if we just used the MC KL estimator:
\begin{equation}
\widehat{\KL}_m(p:q) \eqdef \frac{1}{m} \sum_{i=1}^m    \log \frac{p(x_i)}{q(x_i)},
\end{equation}
we may endup with negative values to our estimated KL, depending on the sample variates!
This never happens for $\eKL$ that is a statistical divergence for the scalar divergence $\ekl(p:q)=p\log \frac{p}{q}+q-p\geq 0$.

Bregman $k$-means~\cite{Bregman-2005,Fast1DClustering-2017} can be applied using either the sided o ther symmetrized centroid~\cite{SymBregman-2009}:
The right-sided centroid is always the center of mass of the parameters.
The left-sided centroid requires to compute $F'(\theta)$ and its reciprocal inverse function $(F'(\theta))^{-1}$ (wlog, assuming $D=1$ for simplicity\footnote{Otherwise, we need to consider monotone operator theory~\cite{SymbolicMonotoneOp-2017} to invert $\nabla F(\theta)$.}).
Although $F'(\theta)$ is available in closed form (and define the dual parameter $\theta$):
\begin{equation}
\tildeG_\calS'(\eta) = \frac{1}{m} \sum_{i=1}^m  \frac{1}{q(x_i)}  \left(p_1(x_i)-p_0(x)\right) \left(1+\log m(x;\eta)\right) =\theta,
\end{equation}
the dual parameter of $(\calM,G)$ cannot be written as a simple function $\eta={F^*}'(\eta)$.
Notice that $\theta=\tildeG_\calS'(\eta)$ is an increasing function of $\eta$ and that inverting operation can be performed numerically.
Indeed, we can compute $\eta=(\tildeG_\calS')^{-1}(\theta)=\tildeG_\calS^*(\theta)$ using a numerical scheme (e.g., bisection search).

The symmetric Jeffreys divergence is:
\begin{eqnarray}
J(m(x;\eta_i):m(x;\eta_j)) &=& \KL(m(x;\eta_i):m(x;\eta_j))+\KL(m(x;\eta_j):m(x;\eta_i)),\\
&=& B_G(\eta_i:\eta_j)+B_G(\eta_j:\eta_i),\label{eq:JBreg}\\
&=& B_G(\eta_i:\eta_j)+B_{G^*}(\theta_i:\theta_j),\\
&=&\inner{\Delta\theta_{ij}}{\Delta\eta_{ij}},
\end{eqnarray}
where $\Delta\theta_{ij}=\theta_i-\theta_j$ and $\Delta\eta_{ij}=\eta_i-\eta_j$.

We may approximate the $J$ divergence by considering the Monte Carlo Bregman generator in Eq.~\ref{eq:JBreg}:
\begin{equation}
J(m(x;\eta_i):m(x;\eta_j)) \simeq B_{\tildeG_\calS}(\eta_i:\eta_j)+B_{\tildeG_\calS}(\eta_j:\eta_i).
\end{equation}
 
We can then apply the technique of mixed Bregman clustering~\cite{MixedBregmanClustering-2008} that considers two centers per cluster.
Moreover a fast probabilistic initialization, called {\em mixed Bregman $k$-means++}~\cite{MixedBregmanClustering-2008}, allows one to guarantee a good initialization with high probability (without computing centroids but requiring to compute divergences).

Another technique to bypass the computation of the gradient $\nabla\tildeG_\calS$ in the BD consists in taking the scaled skew $\alpha$-Jensen divergence~\cite{BR-2011}  for an infinitesimal value of $\alpha$.
Indeed, we have the $\alpha$-Jensen divergence defined by:
\begin{equation}\label{eq:jf}
J_F^\alpha(p:q)=(1-\alpha)F(p)+\alpha F(q)-F((1-\alpha)p+\alpha q), 
\end{equation}
and asymptotically this skewed Jensen divergences yield the sided Bregman divergences~\cite{BR-2011} as follows:
\begin{eqnarray}
\lim_{\alpha\rightarrow 0^+} \frac{J_F^\alpha(p:q)}{\alpha}&=& B_F(q:p),\\
\lim_{\alpha\rightarrow 1^-} \frac{ J_F^\alpha(p:q)}{1-\alpha}&=& B_F(p:q),
\end{eqnarray}

Thus we have for small values of $\alpha>0$ (say, $\alpha=0.001$):
\begin{eqnarray}
J(m(x;\eta_i):m(x;\eta_j)) &=& B_G(\eta_i:\eta_j)+B_G(\eta_j:\eta_i),\\
&\simeq & \frac{1}{\alpha} J_{\tildeG_\calS}^\alpha(\eta_i:\eta_j) +  \frac{1}{1-\alpha} J_{\tildeG_\calS}^{1-\alpha}(\eta_i:\eta_j).\label{eq:symsJ}
\end{eqnarray}

The last equation Eq.\ref{eq:symsJ} is the symmetrized skew Jensen divergence studied in~\cite{nielsen-2010}.

Figure~\ref{fig:clusteringgmm} plots the result of a $2$-cluster clustering wrt the Jeffreys' divergence for a set of $n=8$ mixtures.

 %

\begin{figure}%
\centering
 
\includegraphics[width=0.85\columnwidth]{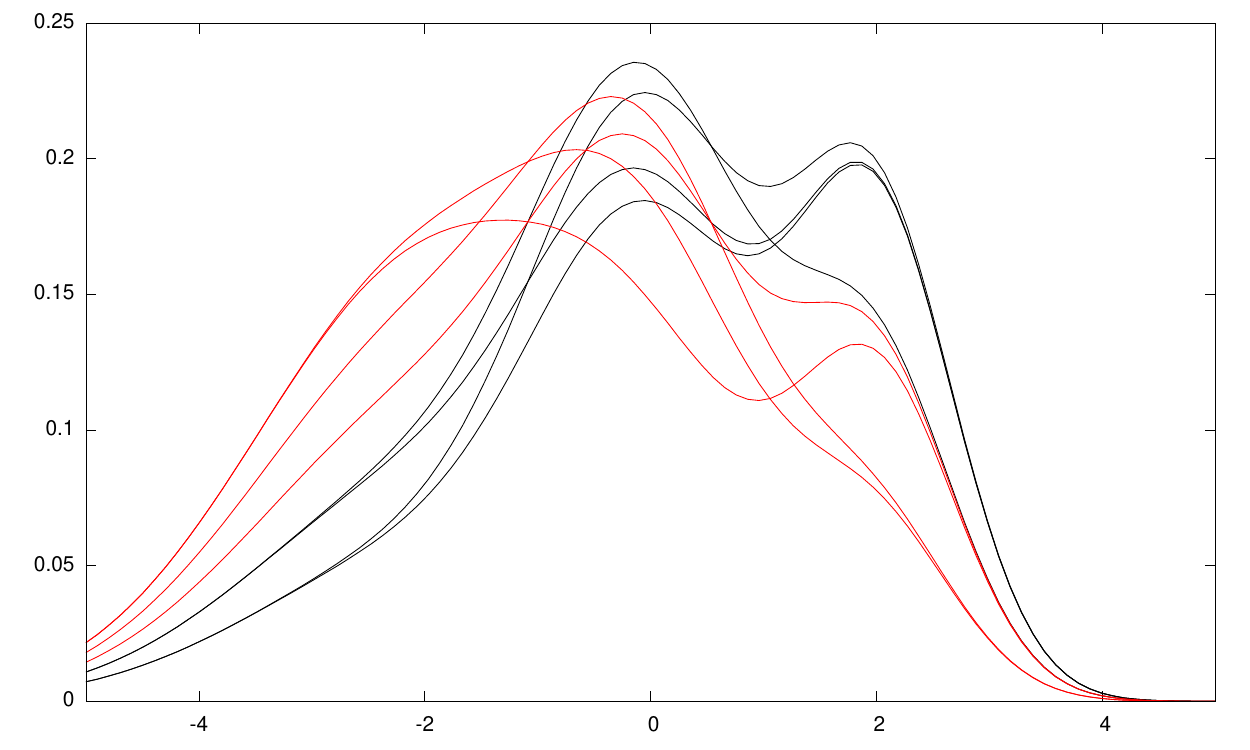}%

\caption{Clustering a set of $n=8$ statistical mixtures of order $D=2$ with $K=2$ clusters: 
Each mixture is represented by a 2D point on the mixture family manifold.
The Kullback-Leibler divergence is equivalent to an integral-based Bregman divergence that is computationally untractable: The Bregman generator is stochastically approximated by Monte Carlo sampling.}%
\label{fig:clusteringgmm}%
\end{figure}

\subsection{Parallelizing information geometry}\label{sec:parallel} 
We can distribute the Monte Carlo information geometry either on a multicore machine with $l$ cores with shared memory or on a cluster of $l$ machines with distributed memory, or even consider hybrid architectures.

Let $(M,\tildeF_{\calS_1}), \ldots, (M,\tildeF_{\calS_l})$ be a set of $l$ information-geometric manifolds obtained from iid sample sets $\calS_1,\ldots,\calS_l$. 
Let $\oplus_{i=1}^s \calS_i$ be a partition of $\calS$.

\subsubsection{Multicore architectures}

On a multicore architecture, we may evaluate the mixture family Bregman divergence $B_{\tildeG_\calS}(\eta:\eta')$ by evaluating 
$B_{\tildeG_{\calS_i}}(\theta:\theta')$, and using the compositionality rule of Bregman generators in BDs   (Property~\ref{prop:bdcomprule}) with:
\begin{equation}\label{eq:distmc}
\tildeG_{\calS}(\theta)= \sum_{i=1}^l \frac{|\calS_i|}{|\calS|} \tildeG_{\calS_i}(\eta).
\end{equation}
That is, $\tildeG_{\calS}(\eta)$ is the {\em arithmetic weighted mean} of the mixture sub-generators.

For the exponential families, recall that we have:
\begin{equation}\label{eq:}
\tildeF_{\calS}(\theta)=\log \left( \sum_{i=1}^s \frac{|\calS_i|}{|\calS|} \exp(\tildeF_{\calS_i}) \right).
\end{equation}
That is, $\tildeF_{\calS}(\theta)$ can be interpreted as an {\em exponential mean} (quasi-arithmetic mean, called $f$-mean~\cite{BR-2011} for the monotonically increasing function $f(x)=\exp(x)$) of the sub-generators.
Thus we can perform the computation of the MC Bregman generators on multi-core architectures easily with a MapReduce strategy~\cite{hpc-2016}.

\begin{fact}[MapReduce evaluation of MC Bregman generators]
The MCMF or MCEF functions can be computed in parallel using a quasi-arithmetic mean MapReduce operation.
\end{fact}

\subsubsection{Cluster architectures}

Since the MC Bregman generators can be interpreted as  random variables $\tildeG_{m}(\theta)$ and $\tildeF_{m}(\theta)$, 
we may obtain robust estimate~\cite{Pelletier-2005} by carrying the calculations on $l$ MCIGs on a cluster architecture, and then integrate those $l$ geometries.

Given a sequence of matching parameters $\theta_1\in (M,\tildeF_{s_1}),\ldots, \theta_l\in (M,\tildeF_{s_l})$, we aggregate these parameters by doing the {\em KL-averaging} method~\cite{KLaveraging-2014}. This amounts to compute a sided centroid for $\theta$.

\section{Conclusion and perspectives}\label{sec:concl}


In this work, we have proposed a new type of {\em randomized
  information-geometric structure} to cope with computationally
untractable information-geometric structures (types 4 and 5 in the
classification of Table~\ref{tab:level}): Namely, the Monte Carlo Information Geometry (MCIG). MCIG  performs stochastic integration of the ideal but computationally intractable definite integral-based Bregman generator (e.g. Eq~\ref{eq:negent} for mixture family) for mixture family and Eq~\ref{eq:cf} for exponential family).
 We proved that the MC Bregman generators for the mixture family and the exponential family are almost surely strictly convex and differentiable (Theorem~\ref{thm:MCMF} and Theorem~\ref{thm:MCEF}, respectively), and  therefore yields a computational tractable information-geometric structure (type 2 in the classification of Table~\ref{tab:level}).
Thus we can get a series of {\em consistent} and {\em computationally-friendly} information-geometric structures that tend asymptotically to the untractable ideal information geometry.
We have demonstrated the usefulness of our technique for a basic Bregman $k$-means clustering technique: Clustering statistical mixtures on a mixture family manifold.
Although the MCIG structures are computationally convenient, we do not have in closed-form $\nabla F^*$ (nor $F^*$) because our Bregman generators are the sum of basic generators whose gradients is the sum of elementary gradients that cannot be inverted easily. This step requires a numerical or symbolic technique~\cite{SymbolicMonotoneOp-2017}.

We note that in the recent work of~\cite{matsuzoe-2017},  Matsuzoe et al. defined a sequence of statistical manifolds relying on a sequential structure of escort expectations for non-exponential type statistical models. 

In a forthcoming work~\cite{MCIG-general-2018}, we address the more general case of the Monte Carlo information-geometric structure of a generic statistical manifold of a parametric family of distributions induced by an arbitrary statistical $f$-divergence.
That is, we consider a statistical divergence $D(p(x;\theta_1):p(x;\theta_2))=\int_{x\in\calX} D_1(p(x;\theta_1):p(x;\theta_2))\dmu(x)$ (where $D_1(\cdot:\cdot)$ is a univariate divergence), and study the information-geometric structure 
$(\calM,g_{\tildeD},\nabla_\tildeD,\nabla_\tildeD^*)$ induced by the Monte Carlo stochastic approximation of the  divergence with $m$ iid samples $x_i$'s: $\tildeD(p(x;\theta_1):p(x;\theta_2)) \eqdef \frac{1}{m} \sum_{i=1}^m \frac{1}{q(x_i)} D_1(p(x_i;\theta_1):p(x_i;\theta_2))$.

\vskip 0.5cm
Codes for reproducible results are available at: \\
\centerline{\url{https://franknielsen.github.io/MCIG/}}


\appendix

\section{$\lse_0^+(x)$ is a Bregman generator\label{sec:proof-lse0p}}
We give the proof of Theorem~\ref{thm:lse0pBF}:

\begin{proof}
Since $\lse_0^+(x_1,\ldots, x_d)=\log \left(1+ \sum_{i=1}^d \exp(x_i)\right)$ is twice continuously differentiable, it suffices to prove that $\nabla^2\lse_0^+(x)\succ 0$.
We have:
\begin{eqnarray}
\partial_i \lse_0^+(x) &=& \frac{e^{x_i}}{1+\sum_k e^{x_k}},\\
\partial_j\partial_i \lse_0^+(x) &\stackrel{j\not = i}{=}&  \frac{-e^{x_i} e^{x_j} }{(1+\sum_k e^{x_k})^2},\\
 \partial_i\partial_i \lse_0^+(x) &=&  \frac{e^{x_i}(1+\sum_k e^{x_k}) -e^{x_i} e^{x_j} }{(1+\sum_k e^{x_k})^2}.
\end{eqnarray}

It follows that the Hessian $(\partial_j\partial_i \lse_0^+(x))_{ij}$ is a diagonally dominant matrix since:
\begin{equation}
e^{x_i}\left(1+\sum_k e^{x_k}\right) = e^{x_i} + e^{x_i}\sum_k e^{x_k} > \sum_{j\not =i} \left|-e^{x_i} e^{x_j} \right| = e^{x_i}  \sum_{j\not =i} e^{x_j}.
\end{equation}

To conclude that the Hessian matrix is SPD, we use Gershgorin circle theorem~\cite{Gervsgorin-2010} to bound the spectrum of a square matrix: 
The eigenvalues of the Hessian matrix are thus real and fall inside a disk of center $(e^{x_i}(1+\sum_k e^{x_k}))_i$ and radius $e^{x_i}  \sum_{j\not =i} e^{x_j}$. Therefore all eigenvalues are positive, and the Hessian matrix is positive definite.

\end{proof}

For $x=(x_1,\ldots, x_d)\in\bbR^d$, we have:
\begin{equation}
 \nabla\lse(x) = \sigma(x),
\end{equation}
where $\sigma(x)$ is the {\em softmax} function:
\begin{equation}
\sigma(x) \eqdef \left( \frac{e^{x_i}}{\sum_{k=1}^d e^{x_k}} \right)_{i\in\{1,\ldots, d\}}.
\end{equation}

Thus by analogy, we may define for $x\in\bbR^d$:
\begin{equation}
\sigma_0^+(x) \eqdef \left( \frac{e^{x_i}}{1+\sum_k e^{x_k}} \right)_{i\in\{1,\ldots, d\}},
\end{equation}
so that $\nabla \lse_0^+(x) =\sigma_0^+(x)$.

\end{document}